\documentclass[journal,letterpaper]{IEEEtran}

\usepackage{microtype}
\usepackage{graphicx}
\usepackage{subfigure}
\usepackage{wrapfig}

\usepackage[sort]{cite}

\usepackage[utf8]{inputenc} % allow utf-8 input
\usepackage[T1]{fontenc}    % use 8-bit T1 fonts
\usepackage{hyperref}       % hyperlinks
\usepackage{url}            % simple URL typesetting
\usepackage{booktabs}       % professional-quality tables
\usepackage{amsfonts}       % blackboard math symbols
\usepackage{nicefrac}       % compact symbols for 1/2, etc.
\usepackage{microtype}      % microtypography
\usepackage{xcolor}         % colors

\usepackage{amsmath,amssymb,mathtools,amsthm}

\newtheorem{theorem}{Theorem}

\newtheorem{lemma}[theorem]{Lemma}

\newtheorem{definition}[theorem]{Definition}

\theoremstyle{remark}
\newtheorem{remark}{Remark}

\DeclareMathOperator*{\argmax}{arg\,max}

\usepackage{dsfont}

\newcommand{\Eb}{\mathbb{E}}
\newcommand{\Pb}{\mathbb{P}}

\newcommand{\oneb}{\mathds{1}}
\newcommand{\Bc}{\mathcal{B}}

\newcommand{\Ec}{\mathcal{E}}

\newcommand{\Hc}{\mathcal{H}}
\newcommand{\Ic}{\mathcal{I}}

\usepackage{diagbox}
\usepackage{enumitem}

\usepackage{algorithm}
\usepackage{algpseudocode}

\usepackage{tabularx}

\allowdisplaybreaks

\begin{document}

\title{Reward Teaching for Federated Multi-armed Bandits}

\author{Chengshuai~Shi, 
        Wei~Xiong, 
        Cong~Shen,~\IEEEmembership{Senior Member,~IEEE,}
        and Jing~Yang,~\IEEEmembership{Senior Member,~IEEE}
\thanks{A preliminary version \cite{shi2023isit} of this work was presented at the 2023 IEEE International Symposium on Information Theory (ISIT).}
\thanks{The work of CSs was supported in part by the US National Science Foundation (NSF) under awards 2029978, 2143559, 2002902, Virginia Commonwealth Cyber Initiative, and the Bloomberg Data Science Ph.D. Fellowship. The work of JY was supported in part by the US NSF under awards 2030026, 2114542, and 1956276.}
\thanks{{C. Shi and C. Shen are with the Charles L. Brown Department of Electrical and Computer Engineering, University of Virginia, Charlottesville, VA 22904, USA. W. Xiong is with the Department of Computer Science, University of Illinois Urbana-Champaign, Urbana, IL 61801. J. Yang is with the Department of Electrical Engineering, The Pennsylvania State University, University Park, PA 16802, USA. E-mail: \texttt{\{cs7ync,cong\}@virginia.edu, wx13@illinois.edu, yangjing@psu.edu.}}}
}

\maketitle

\begin{abstract}
Most of the existing federated multi-armed bandits (FMAB) designs are based on the presumption that clients will implement the specified design to collaborate with the server. In reality, however, it may not be possible to modify the clients' existing protocols. To address this challenge, this work focuses on clients who always maximize their individual cumulative rewards, and introduces a novel idea of ``reward teaching'', where the server guides the clients towards global optimality through implicit local reward adjustments. Under this framework, the server faces two tightly coupled tasks of bandit learning and target teaching, whose combination is non-trivial and challenging. A phased approach, called \emph{Teaching-After-Learning (TAL)}, is first designed to encourage and discourage clients' explorations separately. General performance analyses of TAL are established when the clients' strategies satisfy certain mild requirements. With novel technical approaches developed to analyze the warm-start behaviors of bandit algorithms, particularized guarantees of TAL with clients running UCB or $\varepsilon$-greedy strategies are then obtained. These results demonstrate that TAL achieves logarithmic regrets while only incurring logarithmic adjustment costs, which is order-optimal w.r.t. a natural lower bound. As a further extension, the \emph{Teaching-While-Learning (TWL)} algorithm is developed with the idea of successive arm elimination to break the non-adaptive phase separation in TAL. Rigorous analyses demonstrate that when facing clients with UCB1, TWL outperforms TAL in terms of the dependencies on sub-optimality gaps thanks to its adaptive design. Experimental results demonstrate the effectiveness and generality of the proposed algorithms.
\end{abstract}

\section{Introduction}\label{sec:intro}
Federated multi-armed bandits (FMAB) \cite{zhu2020federated,shi2021federated,reddy2022almost,huang2021federated,dubey2020differentially,li2021asynchronous} is a recently proposed framework that introduces the core principles of federated learning (FL) \cite{mcmahan2017communication,zhang2021fedpd} into multi-armed bandits (MAB) \cite{lattimore2020bandit,tekin2015distributed, liu2010distributed}. In particular, FMAB often considers a system of one global server and multiple \emph{heterogeneous} local clients with the goal of having the clients converge to the \emph{global} optimality. Since proposed by \cite{shi2021federated,zhu2020federated}, FMAB has found applications in cognitive radio, recommender systems, and beyond. 

One practical difficulty of realizing FMAB is that the existing designs have to implement new protocols for both the server and clients \cite{shi2021federated,reddy2022almost,li2022privacy}. Specifically, the server and clients must strictly follow the design collaboratively. In real-world applications, it is relatively easy to update the server's protocols for FMAB. However, given the typically large number of clients, it is often not realistic to assume that all of their protocols can be updated due to infrastructure cost and complicated agent behaviors. 

We first use the example of {\it cognitive radio systems}, a common motivating application for FMAB \cite{shi2021federated,karpov2022collaborative,demirel2022federated}, for a more concrete illustration. Specifically, the base station (i.e., the central server) wants to find a good channel to broadcast information to mobile devices in its coverage area. However, different mobile devices, which are modeled as clients in FMAB, typically have different local channel availabilities due to their different geographic locations. As aforementioned, previous designs (e.g., \cite{shi2021federated, zhu2020federated}) typically require mobile devices (i.e., clients) to follow the new FMAB protocols to collaborate with the base station. However, in reality, mobile devices are often configured to optimize their individual communication qualities following their built-in protocols. It is typically hard and expensive to update all mobile devices to follow the new FMAB designs, especially since such changes are often needed for both software and hardware. {Moreover, in the recommender system, another well-accepted application of FMAB \cite{shi2021federated,shi2021federated2,reddy2022almost,yang2021cooperative,chen2022federated}, the online sellers (i.e., clients) often need to select items (i.e., actions) for promotions on the shopping platform (i.e., the server). However, these sellers typically follow their own strategies to optimize profits and often ignore other social influences, such as environmental effects and health concerns (e.g., for cigarettes). It is thus unrealistic to assume that the selfish sellers would strictly perform the previously proposed FMAB designs.}

This work removes this limitation for FMAB by \emph{designing mechanisms only at the server side}. Especially, the clients can still follow the original routines to optimize their individual performances (as in the aforementioned examples of cognitive radio {and recommender systems}) and \emph{no change of their protocols is required}. Towards this end, a novel \textbf{``reward teaching''} approach is proposed: the server implicitly adjusts the local rewards perceived by the clients to influence their decision-making indirectly. We note that this idea is practical for the aforementioned {applications. For cognitive radio, it} is widely adopted in standard communication protocols for the base station to measure rewards (e.g., throughput) and send designed signals to mobile devices. {In recommender systems, the bonuses received by the sellers are commonly designed and distributed by the shopping platform.}

From a different perspective, this work can also be viewed as breaking the barrier of \emph{naive clients} in the previous FMAB designs, where the clients unconditionally follow the server's instructions. Such naive behaviors are often unrealistic, while a more reasonable scenario (as in this work) is that the clients take actions to optimize their local performances, which may not always align with the server's global objective.

Note that the seemingly simple idea of reward adjustment brings considerable challenges for the server strategy. In particular, the server needs to determine how to adjust rewards to handle the following two tasks \emph{simultaneously}: \textbf{bandit learning} and \textbf{target teaching}. On one hand, the server has to learn the \emph{unknown} global model through the clients' actions, which are based on local observations and may not align with the server's global objective. Thus, reward adjustments should be carefully placed to have the clients explore with respect to (w.r.t.) the global information (instead of their local ones). On the other hand, even if the global model is learned successfully, the corresponding learning history has a cumulative effect on guiding the clients towards the learned target, as all historical (adjusted) rewards are considered by the client in her future decision-making. As a result, while having been studied individually (e.g., learning in MAB and teaching in data-poisoning MAB), the combination of these two tasks is novel and challenging as they are \textit{tightly coupled}, which is detailed in Sec.~\ref{sec:warm}. 

The contributions of this work are summarized as follows.

$\bullet$ \textbf{A reward-teaching framework.} A novel idea of reward teaching is proposed to let the server design reward signals to guide clients with their own local strategies. This idea is practically appealing for FMAB systems as existing client protocols do not have to be modified -- only the reward signals they receive are adjusted. From another perspective, it also provides a method to handle non-naive FMAB clients.

$\bullet$ \textbf{Client strategy-agnostic algorithm designs.} A phased approach, coined \emph{``Teaching-After-Learning'' (TAL)}, is proposed. It addresses the challenge of teaching in an unknown environment by separately encouraging and discouraging explorations in two phases. A more adaptive \emph{``Teaching-While-Learning'' (TWL)} algorithm is then developed to break the strict two-phased structure via the idea of \emph{successive arm elimination}. It is worth noting that {both TAL and TWL are agnostic to the clients' local strategies}.

$\bullet$ \textbf{Client strategy-dependent analysis.} When the clients' local strategies satisfy some general properties, theoretical regret and cost guarantees of TAL are established. Particularizing these properties to UCB1 and $\varepsilon$-greedy \cite{auer2002finite} strategies at clients reveals that TAL can achieve a logarithmic regret while only incurring a logarithmic adjustment cost, which is order-optimal w.r.t. a natural lower bound. Regarding TWL, its advantage is rigorously established with clients running UCB1, where TWL achieves an improved performance dependency on the sub-optimality gaps than TAL due to its adaptive design. Moreover, one key ingredient to obtain these results is the novel technical approaches developed to analyze the \emph{warm-start} behaviors of bandit algorithms, which may be of independent merit. 

$\bullet$ \textbf{Experimental results.} The performance of the proposed designs is verified empirically. Especially, their effectiveness and generality are corroborated with different client strategies (i.e., UCB1, $\varepsilon$-greedy, Thompson sampling \cite{thompson1933likelihood}, and their mixtures), where the advantage of TWL is also evidenced.

\section{Related Works}\label{sec:related}
\textbf{FMAB.} FMAB can be viewed as a variant of the general problem of multi-agent bandits \cite{tekin2015distributed, liu2010distributed, gai2014distributed, shahrampour2017multi, shi2021multi}, where global rewards instead of local ones measure the performance. Recent studies have investigated its robustness \cite{mitra2021exploiting}, personalization \cite{shi2021federated2} and privacy protection \cite{li2022privacy}, and extended the studies to contextual bandits \cite{dubey2020differentially, li2021asynchronous, huang2021federated}. However, almost all of the previous studies assume the clients follow updated local protocols, which either require clients to directly follow the server's instructions or have them work collaboratively. Instead, the designs in this work are purely on the server's side and no change is needed on the client side, which broadens the applicability of FMAB.

\textbf{Reward adjustments in MAB and RL.} 
One line of research on reward adjustments focuses on the malicious poisoning attacks \cite{ma2018data, garcelon2020adversarial,yang2020adversarial}.
The most relevant works are under the ``strong attack'' model \cite{jun2018adversarial,liu2019data,rangi2021secure, zuo2020near}, where the attacker perturbs the rewards \emph{after} observing the player's actions and tricks her into converging to a pre-selected sub-optimal arm (see Sec.~\ref{sec:warm}). Other forms of attacks are also studied \cite{feng2020intrinsic,liu2020action,kapoor2019corruption,guan2020robust,shi2021jsait}, including the ``weak attack'' model \cite{lykouris2018stochastic,gupta2019better,liu2021cooperative} where attacks are performed \emph{before} observing actions. 
Note that the attackers in all these works have no desire to explore the environment, while the reward-teaching server has to actively learn the global model.

Another line is more \emph{conceptually} related to this work: performing adjustments for positive purposes, such as reward shaping \cite{ng1999policy,laud2004theory,csimcsek2006intrinsic}. Especially, in reward shaping, the goal is to accelerate learning using a newly designed set of rewards; thus the optimal policy is kept the same. However, for reward teaching, the goal is to use modified rewards to guide clients to a different optimal policy (i.e., the optimal global arm). The recent work by \cite{zhang2021sample} shares a similar idea of ``teaching'' the player via certain adjustments in reinforcement learning (RL); however, the target is still pre-selected. While differences exist between these previous attempts and this work, they all demonstrate the potential of ``teaching'' in MAB and RL.

In addition, the reward-teaching idea shares similarities with the design of implicit rewards in hierarchical RL \cite{ghavamzadeh2006hierarchical, pateria2021hierarchical}. Thus, the designs in this work may contribute to improving the theoretical understanding of hierarchical RL, which is currently lacking. In particular, our work may be useful in demonstrating that client behaviors can be guided via a small number of modifications on their original rewards.

\textbf{Incentivized explorations in MAB and RL.} Another related research domain is the incentivized explorations in MAB and RL. Especially, a principal leverages either strategically designed signals \cite{mansour2015bayesian,mansour2016bayesian,slivkins2021exploration} or additional compensations \cite{frazier2014incentivizing, han2015incentivizing,wang2018multi} 
to motivate the agent to perform certain actions. In particular, \cite{shi2021almost} leverages additional bonuses to motivate non-naive FMAB clients to perform certain explorations. However, comparing incentivized explorations with this work, we note that major differences exist: the incentivizing principal's signals or compensations are \emph{explicit} to the agent, who then takes corresponding actions; however, the reward adjustment used by the server in this work is \emph{implicit} to the clients, who autonomously perform their own local strategies.

\begin{table}[thb]
    \centering
    \caption{{A summary of key notations used in this work.}}
    \label{tab:my_label}
    {
    \begin{tabularx}{\linewidth}{cX}
    \hline
    \textbf{Notations} & \textbf{Explanations}\\
    \hline
      $M$   & The number of clients and local models \\
       $K$  & The number of available arms\\
       $X_{k,m}(t)$ & The reward of arm $k$ on local model $m$ at time $t$ \\
       $Y_k(t)$ & The reward of arm $k$ on the global model at time $t$\\
        $\mu_{k,m}$ & The expected reward of arm $k$ on local model $m$ \\
       $\nu_k$ & Th expected reward of arm $k$ on the global model\\
       $k_{*,m}$ & The optimal arm for local model $m$ \\
       $k_{\dagger}$ & The optimal arm for the global model \\
       $\mu_{*,m}$ & The expected reward of the locally optimal arm $k_{*,m}$ on local model $m$ \\
       $\mu_{\dagger, m}$ & The expected reward of the globally optimal arm $k_{\dagger}$ on local model $m$\\
       $\nu_{\dagger}$ & The expected reward of the globally optimal arm $k_{\dagger}$ on the global model \\
       $X'_{\pi_m(t),m}(t)$ & The modified observation  for client $m$'s action $\pi_m(t)$ at time $t$\\
       $\sigma_m(t)$ & The adjustment amount performed on client $m$'s observation at time $t$\\
       $R_m(T)$ & The cumulative global regret caused by client $m$\\
       $R_F(T)$ & The cumulative global regret caused by all clients\\
       $C_m(T)$ & The cumulative cost for adjusting client $m$'s observations\\
       $C_F(T)$ & The cumulative cost for adjusting all clients' observations\\
       $\Delta_k$ & The suboptimality gap of arm $k$, i.e., $\nu_{\dagger} - \nu_k, \forall k\neq k_{\dagger}$\\
       $\Delta_{\min}$ & The minimal suboptimality gap, i.e., $\min_{k\neq k_{\dagger}} \Delta_k$\\
       $\Delta_{\max}$ & The maximal suboptimality gap, i.e., $\max_{k\neq k_{\dagger}} \Delta_k$\\
    \hline
    \end{tabularx}}
\end{table}

\section{Problem Formulation}\label{sec:problem}
\subsection{Federated Multi-armed Bandits}\label{subsec:fmab_formulation}
\textbf{Local and global models.} Following \cite{shi2021federated, reddy2022almost,zhu2020federated}, a standard FMAB system of $M$ local models and one global model is considered. With the same set of $K$ arms shared by all the models, at each time step $t \in [T]$, each arm $k\in[K]$ is associated with a local reward $X_{k,m}(t)\in[0,1]$ for each local model $m\in [M]$ and a global reward $Y_{k}(t)\in[0,1]$ for the global model. These rewards of each arm $k$ are all independently sampled with unknown expectations denoted as $\mu_{k,m} := \Eb[X_{k,m}(t)], \forall m\in [M]$ and $\nu_{k} := \Eb[Y_{k}(t)]$. In general, the local arm utilities are model-dependent, i.e., $\mu_{k,m}\neq \mu_{k,n}$ for all $n\neq m$. The optimal local arm for each local model $m$ is denoted as $k_{*,m}:=\argmax_{k\in[K]}\mu_{k,m}$ with $\mu_{*,m} :=\mu_{k_{*,m},m}$, and the optimal global arm as $k_{\dagger}:=\argmax_{k\in[K]}\nu_{k}$ with $\nu_{\dagger}:=\nu_{k_{\dagger}}$.

As in \cite{zhu2020federated,shi2021federated, reddy2022almost}, we consider the setting where each arm $k$'s mean reward on the global model is the average of its mean rewards on the local models\footnote{Other global-local model relationships can also be considered, e.g., the weighted sum in \cite{shi2021federated2}. To better convey the key idea of reward teaching, the exact average, which is simple while representative, is adopted in this work.}, i.e.,
\begin{align*}
    \nu_{k}:= \Eb\left[Y_{k}(t)\right] =  \frac{1}{M}\sum\nolimits_{m\in [M]} \mu_{k,m}.
\end{align*}
As a result, a global-local misalignment may occur as the global optimality may not align with each local optimality, i.e., $k_\dagger\neq k_{*,m}$ for all or part of $m\in [M]$.

\textbf{Clients and server.} 
In FMAB, there exist $M$ clients and one server. At time $t$, each client $m\in [M]$ selects an arm $\pi_m(t)$ (referred to as ``local actions'') and then observes its local reward $X_{\pi_m(t),m}(t)$ on local model $m$. Additionally, each client $m$'s action $\pi_m(t)$ would also generate a reward $Y_{\pi_m(t)}(t)$ from the global model. It would be helpful to interpret the local and global rewards as the individual-level and system-level impact of the clients' actions.

The server in FMAB does not perform any arm-pulling action herself. Instead, she focuses on guiding the local actions to optimize their incurred \emph{global rewards}. However, the global rewards are not directly observable by the server and the clients, which is often a result of practical measurement limitations \cite{shi2021federated}. Instead, the server is assumed to be able to observe the local actions and the corresponding local rewards, i.e., $\{\pi_m(t), X_{\pi_m(t),m}(t):m\in[M]\}$.

To better optimize global performance, previous FMAB studies require that all clients work collaboratively following the updated local protocols. On the contrary, this work considers that clients are fully committed to interacting with their \emph{own} local models (i.e., client $m$ with local model $m$). Then, the clients would naturally adopt their own MAB policies to maximize their local rewards. This setting is practically appealing as in many applications (e.g., the examples of cognitive radio and {recommender systems} in Sec.~\ref{sec:intro}), the local clients are inherently configured to perform local policies to optimize their local performance (e.g., IoT devices maximizing their own data rate and {selfish sellers optimizing their profits}). Specifically, at time $t$, each client $m$ \emph{individually} makes an arm-pulling decision $\pi_m(t)$ based on her own history observed on local model $m$, i.e., $H_m(t-1):= \{\pi_m(\tau), X_{\pi_m(\tau),m}(\tau): 1\leq \tau \leq t-1 \}$.

\subsection{Reward Teaching} \label{sec:rewardteach}
As mentioned, each client $m$ would select suitable actions w.r.t. her own local model, which however may not necessarily meet the server's preference due to the global-local model misalignment. To address this challenge, the following reward-teaching mechanism is introduced for the server to indirectly influence the clients' action selections.

Specifically, after observing $\{X_{\pi_m(t),m}(t):m\in[M]\}$, the server can adjust each client $m$'s local reward $X_{\pi_m(t),m}(t)$ to $X'_{\pi_m(t),m}(t)$ by an amount of $\sigma_m(t)$, i.e.,
\begin{align*}
    X'_{\pi_m(t),m}(t) := X_{\pi_m(t),m}(t)+\sigma_m(t),
\end{align*}
which is then revealed to the client (instead of $X_{\pi_m(t),m}(t)$). Note that one implicit constraint is that the adjusted rewards must still be in $[0,1]$, which is the system limitation.\footnote{In fact, if there is no restriction on the adjustment range, the server is more powerful and the algorithm design is thus easier.} If this constraint is satisfied, the clients are assumed to be unable to detect the reward adjustments by any means. The adjusted rewards lead to an adjusted history of $H'_m(t):=\{\pi_m(\tau),X'_{\pi_m(\tau),m}(\tau):1\leq \tau\leq t\}$ for client $m$, which ideally can shape her future actions in favor of the server. 

It is worth emphasizing that such reward adjustments are practical for FMAB applications. In the cognitive radio example, it is common for the base station to first measure the communication quality (via pilot signals) and then send \emph{designed} feedback to the devices; this is the case in both cellular and WiFi. Adjusting rewards can be achieved via either sending modified feedback signals or modifying the allocated resources (e.g., retransmission bandwidth \cite{Shen2009tc}) to boost or reduce client performance, which is standard in modern communication protocols. The devices, on the other hand, are oblivious to such adjustments thanks to their built-in protocols. {In the application of recommender systems, the shopping platform can implicitly leverage extra or decreased bonuses to guide the decisions of the selfish sellers, e.g., to promote more environmentally friendly and healthier items.}

\begin{figure}[htb]
	\centering
	\includegraphics[width=0.4\textwidth]{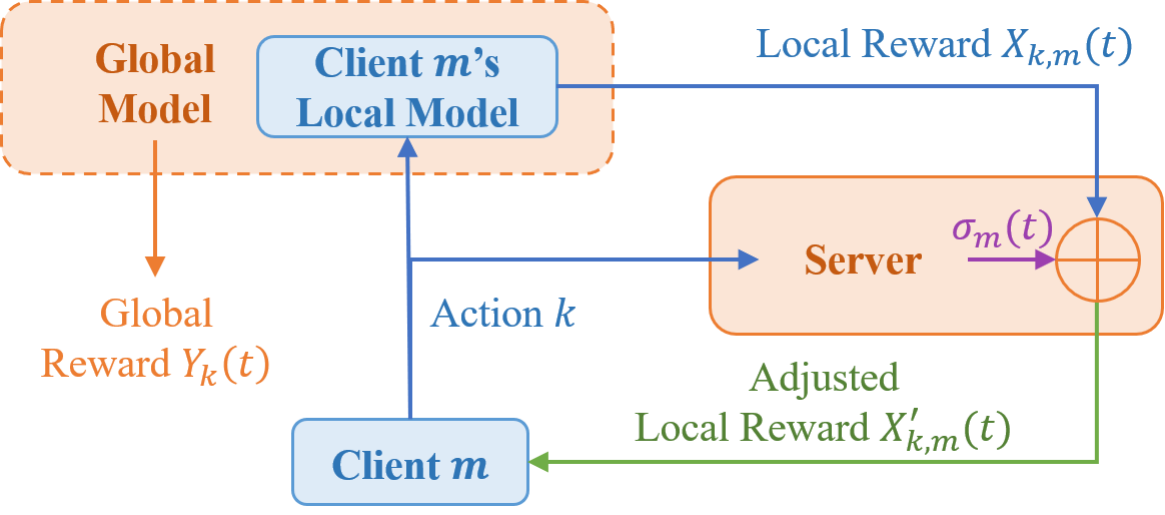}
	\caption{The reward-teaching process with client $m$ (among the overall $M$ clients) and action $\pi_m(t)=k$.}
	\label{fig:model}
\end{figure}

The reward-teaching process is summarized as the following steps, which is also illustrated in Fig.~\ref{fig:model}:

 \setlist{nolistsep}
    \begin{itemize}[noitemsep,leftmargin=*]
        \item Each client $m$ chooses $\pi_m(t)$ using history $H'_m(t-1)$;
        \item The server observes $\{\pi_m(t), X_{\pi_m(t),m}(t):m\in [M]\}$;
        \item The server adjusts $X_{\pi_m(t),m}(t)$ into $X'_{\pi_m(t),m}(t)$ by the amount of $\sigma_m(t)$ for each client $m\in [M]$;
        \item Each client $m$ observes the adjusted $X'_{\pi_m(t),m}(t)$.
    \end{itemize}

\subsection{Learning Objectives}
Following previous FMAB studies, the global view by the server is the focus of our design, which leads to a two-fold objective. First, the server's main goal is to maximize the cumulative \textbf{global} rewards and can be characterized by minimizing the \emph{global regret}, defined as
\begin{align*}
    R_F(T) := \sum\nolimits_{m\in [M]}R_{m}(T),
\end{align*}
where $R_m(T)$ is the regret incurred by client $m$'s actions w.r.t. the global model (instead of her local model) defined as
\begin{align*}
    R_{m}(T):= T\nu_{\dagger} -  \Eb\left[\sum\nolimits_{t\in [T]}Y_{\pi_m(t)}(t)\right].
\end{align*}
The expectation is w.r.t. both the reward generations and the client-system interactions.

Second, the server's adjustments on local rewards are often costly. {For example, in the aforementioned application of cognitive radio, the base station naturally needs to make additional efforts when modifying the originally allocated resources, e.g., infrastructure costs for deviating from the default transmission bandwidth. This work, thus, further introduces the objective of \emph{cumulative cost},} defined as
\begin{align*}
    C_F(T): = \sum\nolimits_{m\in [M]}C_{m}(T),
\end{align*}
where $C_{m}(T)$ denotes the overall cost spent on client $m$ and is further defined as
\begin{align*}
    C_{m}(T): = \Eb\left[\sum\nolimits_{t\in [T]}|\sigma_m(t)|\right].
\end{align*}
The subscripts $F$ in $R_F(T)$ and $C_F(T)$ refer to the global model (i.e., the federation).

Intuitively, there exists a trade-off between these two objectives: with more adjustments on rewards, i.e., larger $C_F(T)$, the server can have a bigger impact on the clients' decisions, which ideally would decrease the regret $R_F(T)$. It is thus important to strike a balance between these two objectives, which is the focus of the remainder of this paper.

% \subsection{Client Strategies}
\subsection{Client Strategies}
\label{sec:clientstra}
To facilitate discussion, we denote client $m$'s local bandit policy as $\Pi_m$. Note that while performing their own policies, the clients are assumed not to be strategically against the server, which is reasonable for most of the real-world applications of FMAB, e.g., autonomous but not fully flexible mobile devices in cognitive radio \cite{shi2021federated}. In addition, we denote  $N_{k,m}(t)$ as the number of pulls by client $m$ on arm $k$ by time $t$, and $N^{-1}_{k,m}(\tau)$ refers to the time step $t$ such that $N_{k,m}(t) = \tau$.

The proposed designs are general and agnostic to clients' strategies, which will be evident in Sections~\ref{sec:tal} and \ref{sec:twl}. For the theoretical analysis, general performance bounds are first provided without specifying the clients' strategies. This is accomplished by identifying the properties of client strategies that lead to the desired theoretical results. More specifically, client-strategy-dependent bounds are then derived (i.e., clients with UCB1 or $\varepsilon$-greedy).  Finally, experiments with varying (and even mixing) strategies for clients are reported.

\section{Two Coupled Tasks and Design Objectives}\label{sec:warm}
In this section, two tightly coupled tasks faced by the reward-teaching server, bandit learning, and target teaching, are elaborated. A system design objective is also proposed.

\textbf{Bandit learning.} One major distinction between learning in FMAB and in classical MAB \cite{bubeck2012regret,lattimore2020bandit} is the server can only gather information through clients' local actions. 
Previous FMAB studies tackled this challenge by implementing new protocols for clients to naively follow \cite{shi2021federated,zhu2020federated, li2022privacy, reddy2022almost}. In contrast, in this work, such information collection can only be indirectly guided via carefully designed rewards. 

\textbf{Target teaching.} To understand teaching, a special case is first considered where the optimal arm $k_{\dagger}$ is known by the server. Then, the goal is to assign adjustments to have the clients pull the \emph{pre-specified} arm $k_{\dagger}$ as much as possible, which is mathematically the same as the \emph{data-poisoning MAB} problem \cite{jun2018adversarial,liu2019data, garcelon2020adversarial, wang2022linear}, where adjustments are phrased as ``attacks''. 
In such scenarios, the server can achieve $R_{m}(T) = O(\log(T))$ and $C_{m}(T) = O(\log(T))$ for each $m\in [M]$ by adjusting rewards from all arms except arm $k_{\dagger}$ to $0$'s \cite{rangi2021secure}.
The underlying philosophy is to ``discourage explorations'' with the adjusted reward $0$'s.

\textbf{Combination leads to a tight coupling.} While both tasks have been separately investigated (to some extent), the reward-teaching server faces a combination of them. On one hand, even if the server can perfectly learn the global model, she still needs to teach it to the clients. On the other hand, to teach correctly, sufficient information must be learned by the server. 
The resulted \emph{tight coupling} is the main challenge of the design. Specifically, the learning attempt has a cumulative effect on teaching, which in return relies on the learned target. Technically, the main resultant difficulty is the analysis of the ``warm-start'' behaviors of bandit algorithms, which is elaborated in Sec.~\ref{sec:tal_analysis}.

\textbf{Design objective.} For the cost, with a known target arm, \cite{rangi2021secure, zuo2020near} prove lower bounds that \textit{with UCB1 and $\varepsilon$-greedy clients (defined in Sec.~\ref{sec:tal_analysis}), it is necessary to spend a cost $C_{m}(T) = \Omega(\log(T))$ to obtain a regret $R_{m}(T)=O(\log(T))$}. Thus, with $M$ \emph{independent} FMAB clients, a cost of $C_F(T) = \Omega(M\log(T))$ is required to obtain a regret of $R_F(T) = O(M\log(T))$ while knowing arm $k_{\dagger}$, which naturally holds for the more stringent case of not knowing the target $k_{\dagger}$.
For the regret, UCB1 and $\varepsilon$-greedy clients can be shown to be conservative \cite{rangi2021secure} as each client $m$ would pull each arm at least $\Omega(\log(T))$ times regardless of the rewards; thus $R_{m}(T)=\Omega(\log(T))$ and $R_F(T)=\Omega(M\log(T))$.

With these results, the following system design goal is established, which is order-wise tight w.r.t. both criteria:\\

\noindent\fbox{%
	\parbox{\linewidth}{%
		\begin{center}
			\emph{\textbf{Goal:} Design algorithms to achieve both $R_F(T)=O(M\log(T))$ and $C_F(T)=O(M\log(T))$.}
		\end{center}
	}%
}
\\

To verify that this goal is non-trivial, two intuitive baseline policies, NG and NA, are discussed as follows, whose limitations are further illustrated experimentally in Sec.~\ref{sec:exp}. 

\noindent $\bullet$ \textbf{``Naively-Guess'' (NG).} The server may randomly initialize one arm $k'$ as the target to adopt the aforementioned approach from \cite{rangi2021secure}. However, the regret would be $R^{\text{NG}}_F(T)=\Omega(MT)$ if $k'\neq k_{\dagger}$, although achieving $C^{\text{NG}}_F(T)=O(M\log(T))$.

\noindent $\bullet$ \textbf{``Naively-Align'' (NA).} Another natural idea is to have the server align $X'_{\pi_m(t),m}(t)$ with $Y_{\pi_m(t)}(t)$ via $\sigma_m(t) = Y_{\pi_m(t)}(t)-X_{\pi_m(t),m}(t)$.\footnote{$Y_{\pi_m(t)}(t)$ is assumed to be observable here for the baseline, which is not the case in our design.} While achieving $R^{\text{NA}}_F(T)=O(M\log(T))$, adjustments would be needed nearly all the time steps, i.e., $C^{\text{NA}}_F(T)=\Omega(MT)$. 

{
\begin{remark}\label{rmk:lower}
    A refined lower bound beyond $\Omega(M\log(T))$ can be instructive, especially for determining the optimal dependencies on parameters other than $M$ and $T$. However, such lower bounds are also challenging, even with a known target \cite{rangi2021secure,zuo2020near}; thus it is left as an open question for future works.
\end{remark}
}

\section{TAL: Algorithm Design}\label{sec:tal}
To address the coupled tasks of bandit learning and target teaching, one idea is to first learn the server's target and then teach the clients to converge to it, which leads to the proposed ``Teaching-After-Learning'' (TAL) algorithm (presented in Alg.~\ref{alg:tal}).
Specifically, TAL starts with the learning phase where the goal is to identify the optimal global arm. Then, in the teaching phase, the server guides the clients toward the learned global optimality. Note that although there is a separation of phases, the teaching phase must handle clients that accumulate observations from the learning phase (i.e., ``warm-start'' clients), whose effect will be more evident in the analysis.

In the learning phase, TAL uniformly adjusts each client $m$'s observed rewards to $\gamma_1$, i.e., $\sigma_m(t) \gets \gamma_1-X_{\pi_m(t),m}(t)$, where $\gamma_1\in [0,1]$ is a to-be-specified input parameter.
Intuitively, this uniform reward adjustment encourages sufficient (or ideally, uniform) explorations among all arms, since their rewards are all at the same value $\gamma_1$. If clients are indeed sufficiently exploring, the server can collect enough information on each arm to identify her optimal arm $k_{\dagger}$. 

This identification is designed to proceed in epochs indexed by counter $\psi$ to ensure statistical independence. If at time $t$, each client $m$ has pulled each arm $k$ at least $F(\psi): = \sum_{\tau\in [\psi]}f(\tau)$ times, where $f(\psi) := \frac{1}{M}\cdot 2^{2\psi+3}\log(2KT^2)$, the server updates upper and lower confidence bounds (UCB and LCB) for each arm $k\in [K]$ using its rewards collected between its $F(\psi-1)+1$ and $F(\psi)$ pulls (i.e., overall $f(\psi)$ pulls) by each client as follows: 
\begin{equation}\label{eqn:f_ulcb}
        \text{UCB}_{k}(\psi), \text{LCB}_{k}(\psi) := \frac{1}{M}\sum\nolimits_{m\in [M]}\hat\mu_{k,m}(\psi) \pm \text{CB}(\psi), 
\end{equation}
where 
\begin{align*}
    \hat\mu_{k,m}(\psi) &:=\sum\nolimits_{\tau = F(\psi-1)+1}^{F(\psi)}X_{k,m}(N^{-1}_{k,m}(\tau))/f(\psi),\\
    \text{CB}(\psi) &:=\sqrt{\log(2KT^2)/(2Mf(\psi))} = 2^{-\psi-2}.
\end{align*} 
Note that with the estimation of $\mu_{k,m}$ from local samples, the first term in Eqn.~\eqref{eqn:f_ulcb} is essentially an estimation $\hat\nu_{k}(\psi)$ of $\nu_{k}$. The confidence bound $\text{CB}(\psi)$ is specifically designed such that  $\text{LCB}(\psi)\leq \nu_{k}\leq \text{UCB}(\psi)$ holds for each arm $k$ and each epoch $\psi$ in the learning phase with high probability. 

The learning phase ends in epoch $\psi$ if the confidence interval of one arm $k_{\ddagger}$ dominates that of all other arms, i.e., $\text{LCB}_{k_{\ddagger}}(\psi)\geq \text{UCB}_{k}(\psi), \forall k\neq k_{\ddagger}$, which is recognized as the optimal arm. Otherwise, a new epoch $\psi+1$ begins. With the designed confidence bound, this identification is guaranteed to be correct with high probability. 

With the identified arm $k_{\ddagger}$, the server utilizes the following adjustments to guide the clients in the teaching phase:
\begin{equation}\label{eqn:TAL_adjustment}
	\sigma_m(t) \gets
	\begin{cases}
		\gamma_2 -X_{\pi_m(t),m}(t) & \text{if $\pi_m(t) \neq k_{\ddagger}$}\\
		0 & \text{if $\pi_m(t) = k_{\ddagger}$}
	\end{cases},
\end{equation}
where $\gamma_2$ is another to-be-specified input parameter and typically should be small. In other words, if the client does not pull arm $k_{\ddagger}$, her reward is adjusted to a small value $\gamma_2$ to discourage explorations; otherwise, the original reward of arm $k_{\ddagger}$ is kept unchanged to save adjustments.

From Alg.~\ref{alg:tal}, it can be observed that TAL is a pure server protocol and agnostic to the clients' local strategies -- the only interaction with the clients is the adjusted rewards.

\begin{algorithm}[thb]
	\small
	\caption{TAL}
	\label{alg:tal}
	\begin{algorithmic}[1]
        \Require Parameter  $\gamma_1, \gamma_2\in [0,1]$; Time Horizon $T$
		\State Initialize: $F\gets 1$ (i.e., the learning phase); $\psi\gets 1$; $k_{\ddagger} \gets 0$
		\For{$t\leq T$} 
		\State Observe $\{\pi_m(t), X_{\pi_m(t),m}(t):m\in[M]\}$
		\If{$F= 1$ \& $N_{k,m}(t)\geq F(\psi), \forall m\in [M],k\in [K]$} 
			\State Update $\{\text{UCB}_{k}(\psi), \text{LCB}_{k}(\psi):k\in [K]\}$ as Eqn.~\eqref{eqn:f_ulcb}
			\If{$\exists j \in [K], \text{LCB}_{j}(\psi)\geq \text{UCB}_{k}(\psi), \forall k\neq j$}
			\State Set $k_\ddagger\gets j$; $F\gets  2$ (i.e., the teaching phase)
            \Else\ Set $\psi \gets \psi +1$
			\EndIf
		\EndIf 
        \If{$F=1$}\ $\sigma_m(t) \gets \gamma_1 -X_{\pi_m(t),m}(t), \forall m\in [M]$
		\ElsIf{$F=2$}  Set $\sigma_m(t)$ as Eqn.~\eqref{eqn:TAL_adjustment}, $\forall m\in [M]$
		\EndIf
        \State Set $X'_{\pi_m(t),m}(t) \gets X_{\pi_m(t),m}(t) + \sigma_m(t), \forall m\in [M]$
		\State Reveal $X'_{\pi_m(t),m}(t)$ to each client $m\in [M]$
		\EndFor
	\end{algorithmic}
\end{algorithm}

\section{TAL: Theoretical Analysis}\label{sec:tal_analysis}
{In this section, we first provide a general analysis of TAL (Theorem~\ref{thm:tal_general}) under some abstract characterizations of clients' strategies (i.e., sufficient-exploring and warm-starting in Definitions~\ref{def:uniform_exploration} and \ref{def:warm_start_pulls}, respectively). Then, we consider clients with UCB1 or $\varepsilon$-greedy in the following two subsections, respectively. In particular, the adopted abstract characterizations are particularized (Lemmas~\ref{lem:tal_learning_UCB}, \ref{lem:tal_teaching_UCB}, \ref{lem:tal_learning_epsilon} and \ref{lem:tal_teaching_epsilon}), and then specific performance guarantees are obtained (Theorems~\ref{thm:tal_UCB} and \ref{thm:tal_epsilon}), which show that TAL achieves the design goals in Section~\ref{sec:warm} with these clients.} Detailed proofs are deferred to Appendix~\ref{app:tal}.

Some useful notations are introduced as follows: $\Delta_{k}:=\nu_{\dagger}-\nu_{k},\forall k\neq k_{\dagger}$, $\Delta_{\min}=\Delta_{k_{\dagger}}:=\min_{k\neq k_{\dagger}}\Delta_{k}$, $\Delta_{\max} := \max_{k\in[K]}\Delta_{k}$, and $\mu_{\dagger,m}:=\mu_{k_{\dagger},m}$. Moreover, $\delta_{k,m}(\gamma) := \Eb[|\gamma - X_{k,m}(t)|]$ and $\psi_{\max}: = \left\lceil\log_2(1/\Delta_{\min})\right\rceil$. Also, without loss of generality, it is assumed that $K, M \ll T$.

We first define sufficiently exploring algorithms for the learning phase in TAL, which states that a bandit algorithm would sufficiently explore when facing uniform rewards.

\begin{definition}[Sufficiently Exploring Algorithms]\label{def:uniform_exploration}
    Consider a $K$-armed bandit environment where rewards from arms in a set $\Ic\subseteq [K]$ are always a fixed constant $\gamma\in [0,1]$. In this environment, a bandit algorithm $\Pi$ is said to be $(\Ic, \gamma, \underline{\eta}, \overline{\eta})$-sufficiently exploring if it would pull each arm in the set $\Ic$ at least $\underline{\eta}(\tau; \gamma, \Ic)$ and at most $\overline{\eta}(\tau; \gamma, \Ic)$ times when total $\tau$ pulls have been performed on set $\Ic$.
\end{definition}
If local strategies are sufficiently exploring as in Definition~\ref{def:uniform_exploration}, enough information can be collected in the learning phase to identify the global optimal arm, as stated in the following lemma, where $\underline{\eta}^{-1}(N; \gamma, [K])$ denotes the value $\tau$ such that $\underline{\eta}(\tau; \gamma, [K]) = N$.
\begin{lemma}[Learning Phase in TAL]\label{lem:tal_learning_general}
     If $\Pi_m$ is $([K], \gamma_{1}, \underline{\eta}_{m}, \overline{\eta}_m)$-sufficiently exploring for all $m\in [M]$, 
    with probability (w.p.) at least $1- 1/T$, the learning phase ends with $k_\ddagger = k_\dagger$ by time step $T_1$, and the regret and cost in the learning phase of TAL are bounded, respectively, as
    \begin{align*}
        R^{\textup{TAL}}_{F,1}(T)\leq &\sum\nolimits_{m\in [M]}\sum\nolimits_{k\neq k_\dagger} \Delta_k \cdot \overline{\eta}_{m}\left(T_1; \gamma_1, [K]\right);\\
        C^{\textup{TAL}}_{F,1}(T)\leq&\sum\nolimits_{m\in [M]}\sum\nolimits_{k\in [K]} \delta_{k,m}(\gamma_1) \cdot \overline{\eta}_{m}\left(T_1; \gamma_1, [K]\right),
    \end{align*}
    where $T_1 \leq \max_{m\in [M]}\{\underline{\eta}_{m}^{-1}\left(F(\psi_{\max}) ; \gamma_1, [K]\right)\}$.
\end{lemma}
{Note that the time step $T_1$ bounded via the sufficiently exploring lower bound (i.e., $\underline{\eta}$) ensures sufficient information collection, while the corresponding upper bound (i.e., $\overline{\eta}$) guarantees performance, i.e., regret and adjustment cost.} 

Then, for the teaching phase, since the cumulative observations from the learning phase are inherited to the client strategies, we can view the clients as ``warm-started''. The following notion of warm-start pulls is introduced, which measures the warm-start behavior of an algorithm.

\begin{definition}[Warm-start Pulls]\label{def:warm_start_pulls}
    In a $K$-armed bandit environment $\Bc$, if a reward sequence $H = \{H_k: k\in [K]\}$ is input to a bandit algorithm $\Pi$, where $H_k$ is a reward sequence for arm $k$, warm-start pulls on arm $k$ is defined as $\iota_k(T; H, \Bc, \Pi) := \Eb_{\Pi}[\sum_{t\in [T]}\oneb\{\pi(t) = k\}|H, \Bc]$, which represents the expected pulls performed by $\Pi$ on each arm $k$ during $T$ steps in environment $\Bc$ with prior input $H$. 
\end{definition}

Using this notion of warm-start pulls, the following guarantee on the teaching phase is established.

\begin{lemma}[Teaching Phase in TAL]\label{lem:tal_teaching_general}
    If the event in Lemma~\ref{lem:tal_learning_general} occurs, the regret and cost in the teaching phase of TAL are bounded, respectively, as
    \begin{align*}
        R^{\textup{TAL}}_{F,2}(T)\leq&\sum_{m\in [M]}\max_{H_m\in \Hc_m}\sum_{k\neq k_\dagger}\Delta_k \cdot \iota_{k}(T; H_m, \Bc_m, \Pi_m);\\
        C^{\textup{TAL}}_{F,2}(T)\leq& \sum_{m\in [M]}\max_{H_m\in \Hc_m}\sum_{k\neq k_\dagger}\delta_{k,m}(\gamma_2)\cdot \iota_{k}(T; H_m, \Bc_m, \Pi_m),
    \end{align*}
    where $\Bc_m$ denotes an environment with constant rewards as $\gamma_2$ for arm $k \neq k_\dagger$ and stochastic rewards with expectation $\mu_{\dagger,m}$ for arm $k_\dagger$. The set $\Hc_m$ is defined with each element of it as a reward sequence $H_m = \{H_{k,m} : k\in [K]\}$ where $H_{k,m}\in \{\{\gamma_1\}^\tau: \tau \in [\underline{\eta}_m(T_1; \gamma_1, [K]), \overline{\eta}_m(T_1; \gamma_1, [K])]\}$.
\end{lemma}
Note that $\Bc_m$ characterizes the environment of client $m$ in the teaching phase while $\Hc_m$ represents the cumulative observation inherited from the learning phase. 

Finally, {the overall performance guarantee can be obtained by combining the regrets from two phases}.
\begin{theorem}[Overall Performance of TAL]\label{thm:tal_general}
    Under the assumption in Lemma~\ref{lem:tal_learning_general}, with $R^{\textup{TAL}}_{F,1}(T), C^{\textup{TAL}}_{F,1}(T)$ defined in Lemma~\ref{lem:tal_learning_general} and $R^{\textup{TAL}}_{F,2}(T), C^{\textup{TAL}}_{F,2}(T)$ in Lemma~\ref{lem:tal_teaching_general}, the regret and cost of TAL are bounded, respectively, as
    \begin{align*}
        R^{\textup{TAL}}_F(T) \leq R^{\textup{TAL}}_{F,1}(T) + R^{\textup{TAL}}_{F,2}(T) + O(M);\\
        C^{\textup{TAL}}_F(T) \leq C^{\textup{TAL}}_{F,1}(T) + C^{\textup{TAL}}_{F,2}(T) + O(M).
    \end{align*}
\end{theorem}
The key difficulty behind this analysis resides in leveraging the quantities in Definitions~\ref{def:uniform_exploration} and \ref{def:warm_start_pulls}. In particular, how to specify $\underline{\eta}$, $\overline{\eta}$ and $\iota$ is non-trivial, which is one of the main technical challenges in proving Thm.~\ref{thm:tal_general}.  Furthermore, Thm.~\ref{thm:tal_general} implies that the desired logarithmic regret and cost can be achieved by TAL when $R^{\textup{TAL}}_{F,1}(T), R^{\textup{TAL}}_{F,2}(T), C^{\textup{TAL}}_{F,1}(T)$ and $C^{\textup{TAL}}_{F,2}(T)$ are all bounded in logarithmic orders. The analyses of these terms are further determined by the sufficiently exploring property and the warm-start pulls of the specific clients' strategies as stated in Lemmas~\ref{lem:tal_learning_general} and \ref{lem:tal_teaching_general}. 

In the following, to particularize the general guarantee in Thm.~\ref{thm:tal_general}, we analyze several well-known bandit algorithms as clients' strategies (i.e., UCB and $\varepsilon$-greedy).

\subsection{UCB Clients}\label{subsec:tal_ucb}
The popular UCB-type algorithms are first considered. In particular, we analyze the celebrated UCB1 algorithm \cite{auer2002finite} while noting that the analysis generalizes to other UCB variants  \cite{audibert2009minimax,garivier2011kl}.  
Especially, at time $t$, the UCB1 algorithm for client $m$ chooses arm as follows:
\begin{align*}
    \pi_m(t)= \argmax_{k\in[K]}\left\{\hat\mu'_{k,m}(t-1)+\sqrt{2\log(t)/N_{k,m}(t-1)}\right\},
\end{align*}
which considers both the perceived sample mean
\begin{align*}
    \hat\mu'_{k,m}(t) := \sum\nolimits_{\tau\in [N_{k,m}(t)]}X'_{k,m}(N^{-1}_{k,m}(\tau))/N_{k,m}(t)
\end{align*}
and the associated confidence bound.

First, the sufficiently exploring assumption in Lemma~\ref{lem:tal_learning_general} is verified for UCB1 in Lemma~\ref{lem:tal_learning_UCB}. This is intuitive as with constant rewards, the sample means are the same while additional pulls decrease the confidence bound in UCB1.
\begin{lemma}\label{lem:tal_learning_UCB}
    For any $\gamma\in [0,1]$ and set $\Ic\subseteq [K]$, UCB1 is $(\Ic, \gamma, \underline{\eta}, \overline{\eta})$-sufficiently exploring with $
        \underline{\eta}(\tau; \gamma, \Ic)= \lfloor \tau/ |\Ic| \rfloor$ and $\overline{\eta}(\tau; \gamma, \Ic) = \lceil \tau/ |\Ic| \rceil$.
\end{lemma}
Then, the performance of TAL in the learning phase (in Lemma~\ref{lem:tal_learning_general}) can be bounded by recognizing $T_1 = O(K\log(T)/(M \Delta_{\min}^2))$, which further specifies the reward sequence set $\Hc_{m}$ in Lemma~\ref{lem:tal_teaching_general} and leads to the following lemma on the warm-start pulls of UCB1.
\begin{lemma}\label{lem:tal_teaching_UCB}
    If $\gamma_1 \geq \mu_{\dagger, m} > \gamma_2$ and $\Pi_m$ is UCB1, for all $k\neq k_{\dagger}$, it holds that $
        \max_{H_m\in \Hc_m}\{\iota_{k}(T; H_m, \Bc_m, \Pi_m)\} = O\left(\frac{(\gamma_1 - \gamma_2 )T_1}{K(\mu_{\dagger,m}-\gamma_2)} + \frac{\log(T)}{(\mu_{\dagger,m} - \gamma_2)^2}\right)$.
\end{lemma}
Proving this lemma is non-trivial and may be of independent interest in understanding the warm-start behavior of UCB1. Essentially, the result can be interpreted as first offsetting the ``warm-start'' history (the first term) and then converging to arm $k_{\dagger}$ (the second term) in an environment $\Bc_m$, whose rewards for arm $k \neq k_\dagger$ are constant $\gamma_2$'s and rewards for arm $k_\dagger$ have an expectation $\mu_{\dagger,m}$  (see Lemma~\ref{lem:tal_teaching_general}).

It is noted that Lemma~\ref{lem:tal_teaching_UCB} requires $\gamma_1 \geq \mu_{\dagger, m}$, which maintains the optimism for the estimation of arm $k_\dagger$ on each local model $m$. The other requirement $\mu_{\dagger, m}> \gamma_2$ is intuitive as otherwise, the local client $m$ would not converge to arm $k_\dagger$. Since there is no prior information about $\mu_{\dagger, m}$, a feasible and sufficient solution is to set $\gamma_1 = 1$ while $\gamma_2 = 0$, which leads to the following theorem.

\begin{theorem}[TAL with UCB1 clients]\label{thm:tal_UCB}
    For TAL with $\gamma_1 = 1$ and $\gamma_2 = 0$, if all clients run  UCB1 locally and $\mu_{\dagger,m}\neq 0$ for all $m\in [M]$, it holds that
   \begin{align*}
		R^{\textup{TAL}}_F(T) &= O\bigg(\sum_{m\in[M]}\sum_{k\neq k_{\dagger}}\bigg[\frac{\Delta_{k}\log(T)}{\mu_{\dagger,m}M\Delta_{\min}^2}+\frac{\Delta_{k}\log(T)}{\mu_{\dagger,m}^2}\bigg]\bigg);\\
		C^{\textup{TAL}}_F(T) &= O\bigg(\sum_{m\in[M]}\sum_{k\in [K]}\frac{(1-\mu_{k,m})\log(T)}{M\Delta^2_{\min}}\\
  &+\sum_{m\in[M]}\sum_{k\neq k_{\dagger}}\bigg[\frac{\mu_{k,m}\log(T)}{\mu_{\dagger,m}M\Delta^2_{\min}}+\frac{\mu_{k,m}\log(T)}{\mu_{\dagger,m}^2}\bigg] \bigg).
	\end{align*}
\end{theorem}
{We note that with a focus on the dependencies on $M$ and $T$, the regret and cost are both of order $O(M\log(T))$; thus TAL is order-optimal w.r.t. both criteria stated in Sec.~\ref{sec:warm}, i.e., the general design goal is achieved. Moreover, the regret bound shows two dominating terms, which are from Lemma~\ref{lem:tal_teaching_UCB}, i.e., the teaching phase.  In fact, there is another non-dominating (thus hidden) term from Lemma~\ref{lem:tal_learning_general} for the learning phase; see more details in Appendix~\ref{app:tal_ucb}. 
A similar three-part form is shared by the cost. In particular, the first term is from the learning phase (thus the sum is over all arms $k\in [K]$ and each term scales with $1-\mu_{k,m}$), and the last two terms are from the teaching phase (thus the sum is over sub-optimal global arms $k\neq k_{\dagger}$ and scales with $\mu_{k,m}$). 
}

\subsection{$\varepsilon$-greedy Clients}\label{subsec:tal_epsilon}
The analysis is further extended to the clients running the $\varepsilon$-greedy algorithm \cite{sutton1998reinforcement}, another well-known bandit strategy. Especially, the $\varepsilon$-greedy algorithm for client $m$ is as follows:
\begin{equation*}
	 \pi_m(t) \gets
	\begin{cases}
		\argmax_{k\in[K]} \hat{\mu}'_{k,m}(t-1) &\text{w.p. $1-\varepsilon_m(t)$}\\
	    \text{a random arm in $[K]$} & \text{w.p. $\varepsilon_m(t)$}
	\end{cases},
\end{equation*}
where the exploration probability $\varepsilon_m(t)\in [0,1]$ is taken as $\varepsilon_m(t) = O(K/t)$,  following \cite{auer2002finite}. 

First, the following lemma states that $\varepsilon$-greedy is sufficiently exploring, which is intuitive as the constant rewards lead to the same sample mean for different arms.
\begin{lemma}\label{lem:tal_learning_epsilon}
    For any $\gamma\in [0,1]$, if ties among arms are broken uniformly at random, with probability at least $1-1/T$, $\varepsilon$-greedy is $([K], \gamma, \underline{\eta}, \overline{\eta})$-uniformly exploring with $\underline{\eta}(\tau; \gamma, [K])$ and $\overline{\eta}(\tau; \gamma, [K]) = O( \tau/K \pm \log(KT))$.
\end{lemma}
Due to the randomness in $\varepsilon$-greedy, it is complicated to analyze its warm-start pulls in general. Instead, the following lemma focuses on $\gamma_1 = \gamma_2 =0$. Under this setting, the sample means are all kept as zero in the learning phase. Thus, once a non-zero reward is collected in the teaching phase, that arm will immediately have the highest sample mean.
\begin{lemma}\label{lem:tal_teaching_epsilon}
    If $\Pi_m$ is $\varepsilon$-greedy and $\mu_{\dagger,m} >\gamma_1 = \gamma_2 = 0$, with probability at least $1-1/T$, it holds that $\max_{H_m\in \Hc_m}\{\sum_{k\neq k_\dagger}\iota_{k,m}(T; H_m, \Bc_m, \Pi_m)\} = O(K\log(KT)/\mu_{\dagger,m}^2)$.
\end{lemma}

Combining these results with Thm.~\ref{thm:tal_general}, the following performance guarantees can be obtained.
\begin{theorem}[TAL with $\varepsilon$-greedy clients]\label{thm:tal_epsilon}
    For TAL with $\gamma_1 = \gamma_2 = 0$, if clients run $\varepsilon$-greedy and break ties uniformly at random, and $\mu_{\dagger,m }\neq 0, \forall m\in [M]$, it holds that
    \begin{align*}
            R^{\textup{TAL}}_F(T) &= O\bigg(\frac{K\Delta_{\max}\log(T)}{\Delta^2_{\min}}+\sum_{m\in[M]}\frac{K\Delta_{\max}\log(T)}{\mu_{\dagger,m}^2}\bigg),\\
            C^{\textup{TAL}}_F(T) &=O\bigg(\sum_{m\in[M]}\bigg[\frac{K\mu_{*,m}\log(T)}{M\Delta^2_{\min}}+\frac{K\mu_{*,m}\log(T)}{\mu_{\dagger,m}^2}\bigg]\bigg).
        \end{align*}
\end{theorem}

{The two parts in regret and cost are from the learning and teaching phases, respectively. It can be observed that TAL with $\varepsilon$-greedy clients also achieves the goal illustrated in Section~\ref{sec:warm}. Moreover, compared with Theorem~\ref{thm:tal_UCB}, dependencies on $\Delta_{\max}$ and $\mu_{*,m}$ (instead of $\Delta_{k}$ and $\mu_{k,m}$) can be observed, which is a worst-case consideration to capture the random actions generated from the $\varepsilon$-greedy policy.}

\subsection{Discussions: Thompson Sampling and Beyond}\label{subsec:tal_analysis_beyond}
Another popular bandit strategy is Thompson sampling (TS) \cite{thompson1933likelihood}. Experiment results in Sec.~\ref{sec:exp} verify the performance of TAL with TS clients; however, the theoretical analysis remains open. In particular, unlike the sufficiently exploring UCB and $\varepsilon$-greedy, \cite{kalvit2021closer} indicates that when facing two arms with constant reward $1$'s, the pulls by TS can be arbitrarily imbalanced. Instead, balanced pulls can be achieved with reward $0$'s for these two arms. This phenomenon motivates using $\gamma_1 = 0$ to encourage TS explorations in the learning phase, whose effectiveness is verified empirically but not analytically. On the other hand, the complicated warm-start behavior of TS also requires further investigation.

{Furthermore, in Secs.~\ref{subsec:tal_ucb} and \ref{subsec:tal_epsilon}, the hyper-parameter $\gamma_1$ is set to different values (i.e., $1$ for UCB clients and $0$ for $\varepsilon$-greedy clients). These choices are made to facilitate the corresponding ``warm-start'' analyses required in Definition~\ref{def:warm_start_pulls} (i.e., to maintain the optimism of estimations in UCB and to avoid complicated analyses due to the randomness in $\varepsilon$-greedy). However, the capabilities of TAL extend beyond these theoretically sound options. Especially, experiments in Sec.~\ref{sec:exp} show that various other choices (e.g., $\gamma_1 = 0$ for UCB clients and $\gamma_1 = 1$ for $\varepsilon$-greedy clients) can also lead to reasonable performances. Thus, it would be an interesting future direction to investigate whether a unified hyper-parameter $\gamma_1$ in TAL is sufficient for certain classes of client strategies (e.g., UCB and $\varepsilon$-greedy). The main difficulty along this direction is still to analyze the ``warm-start'' behaviors, which are largely determined by the specific strategy.}

Moreover, Thm.~\ref{thm:tal_general} has established conditions on clients' strategies to obtain performance guarantees of TAL, i.e., sufficiently exploring and low sub-optimal warm-start pulls. An interesting direction is to verify the client-strategy-agnostic nature of TAL in an even broad sense, e.g., with any no-regret client strategy. Experimental results are provided later to enlighten future works on this open problem.

\section{TWL: A More Adaptive Extension}\label{sec:twl}
\subsection{Algorithm Design}
To further optimize the performance, a more adaptive ``Teaching-While-Learning'' (TWL) algorithm (presented in Alg.~\ref{alg:twl}) is proposed, which breaks the non-adaptive phased structure of TAL by leveraging a different idea of \emph{successive arm elimination} \cite{jamieson2014best,even2002pac}.
In TWL, the server maintains a set $\Upsilon$ of active arms (on the global model), which is initialized as $[K]$. If $|\Upsilon|>1$, the following update is performed after each active arm $k\in \Upsilon$ has been pulled at least $F(\psi)$ times by each client:
\begin{equation*}
	\Upsilon\gets \{j\in \Upsilon:\text{UCB}_{j}(\psi)\geq \text{LCB}_{k}(\psi), \forall k\in \Upsilon\},
\end{equation*}
where $\text{UCB}_{k}(\psi)$ and $\text{LCB}_{k}(\psi)$ are defined in Eqn.~\eqref{eqn:f_ulcb} and  $\psi$ is the epoch counter as in TAL. In this process, the arms that do not satisfy the requirement are eliminated (i.e., marked as inactive). Then, based on the set $\Upsilon$, the following adjustment is performed for client $m$:
\begin{align*}
		\sigma_m(t) \gets
	\begin{cases}
		\gamma_2 -X_{\pi_m(t),m}(t) & \text{if $\pi_m(t) \notin \Upsilon$}\\
  \gamma_1 -X_{\pi_m(t),m}(t)  & \text{if $\pi_m(t) \in \Upsilon$ and $|\Upsilon|>1$}\\
		$0$ & \text{if $\pi_m(t) \in \Upsilon$ and $|\Upsilon|=1$}
	\end{cases},
\end{align*}
where $\gamma_1, \gamma_2 \in [0,1]$ are to-be-specified input parameters.

In other words, the local rewards of all inactive arms are adjusted to $\gamma_2$ (typically small) to discourage explorations. For an active arm, when there are other active arms (i.e., $|\Upsilon|>1$), the server uniformly adjusts its rewards to $\gamma_1$ to encourage explorations. When an arm is the only active one (which is arm $k_\dagger$ with high probability), its original rewards are kept to save server adjustments, which is sufficient as all other arms are inactive with a small perceived reward $\gamma_2$.

TWL is more refined than TAL as it only encourages explorations on the active arms (instead of all arms), which is important in two aspects. First, only necessary arm-dependent explorations are encouraged. Second, fewer cumulative rewards on the sub-optimal arms also alleviate the server's burden of teaching clients to converge to the optimal arm.

\begin{algorithm}[htb]
    \small
	\caption{TWL}
	\label{alg:twl}
	\begin{algorithmic}[1]
        \Require Parameter $\gamma_1, \gamma_2\in [0,1]$; Time Horizon $T$
		\State Initialize: active arm set $\Upsilon\gets [K]$; iteration counter $\psi \gets 1$
		\For{$t\leq T$} 
		\State Observe $\{\pi_m(t), X_{\pi_m(t),m}(t):m\in[M]\}$
		\If{$|\Upsilon|>1$ and $N_{k,m}(t)\geq F(\psi), \forall k \in \Upsilon, m\in [M]$}
		\State Update $\{\text{UCB}_{k}(\psi), \text{LCB}_{k}(\psi): k \in \Upsilon\}$ as in Eqn.~\eqref{eqn:f_ulcb}
        \State Update $\Upsilon\gets \{j\in \Upsilon:\text{UCB}_{j}(\psi)\geq \text{LCB}_{k}(\psi), \forall k\in \Upsilon\}$
        \State Set $\psi \gets \psi+1$
		\EndIf
        \State $\forall m\in [M]$, set
        \begin{equation*}
            \begin{aligned}
            		\sigma_m(t) \gets
            	\begin{cases}
            		\gamma_2 -X_{\pi_m(t),m}(t) & \text{if $\pi_m(t) \notin \Upsilon$}\\
              \gamma_1 -X_{\pi_m(t),m}(t)  & \text{if $\pi_m(t) \in \Upsilon$ and $|\Upsilon|>1$}\\
            		$0$ & \text{if $\pi_m(t) \in \Upsilon$ and $|\Upsilon|=1$}
            	\end{cases},
            \end{aligned}
            \end{equation*}
        \State Set $X'_{\pi_m(t),m}(t) \gets X_{\pi_m(t),m}(t) + \sigma_m(t)$
        %\EndFor
		\State Reveal $X'_{\pi_m(t),m}(t)$ to each client $m\in [M]$
		\EndFor
	\end{algorithmic}
\end{algorithm}

\begin{figure*}[tbh]
	\centering
	\subfigure[UCB1: regrets.]{ \includegraphics[width=0.23\linewidth]{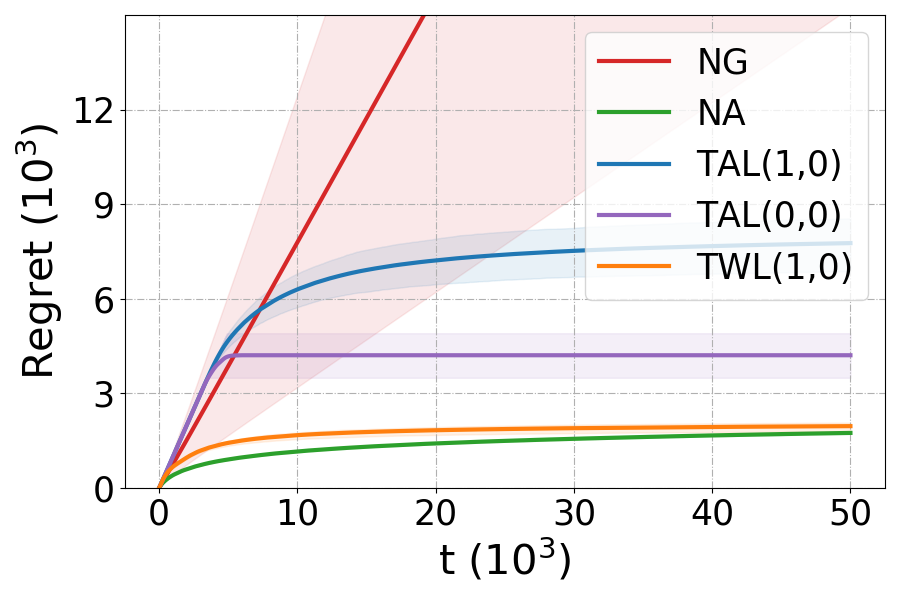}\label{fig:ucb_regret}}
	\subfigure[UCB1: costs.]{ \includegraphics[width=0.23\linewidth]{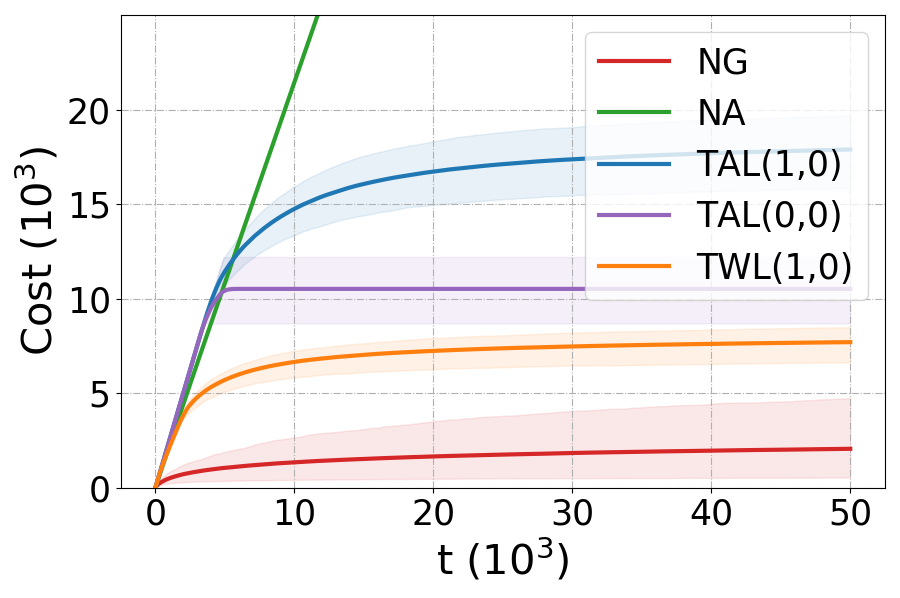}\label{fig:ucb_cost}}
	\vspace{-0.1in}
	\subfigure[$\varepsilon$-greedy: regrets.]{ \includegraphics[width=0.23\linewidth]{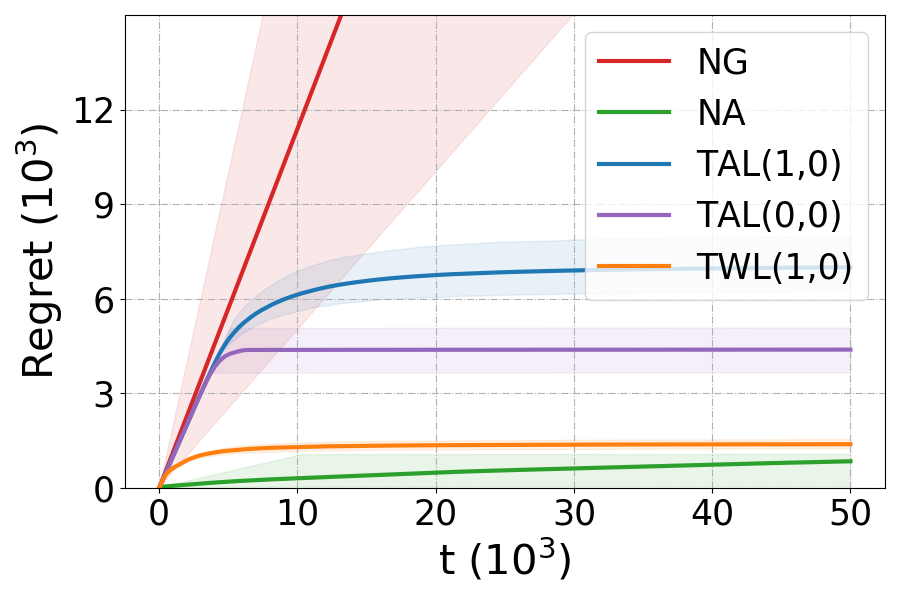}\label{fig:epsilon_regret}}
	\subfigure[$\varepsilon$-greedy: costs.]{ \includegraphics[width=0.23\linewidth]{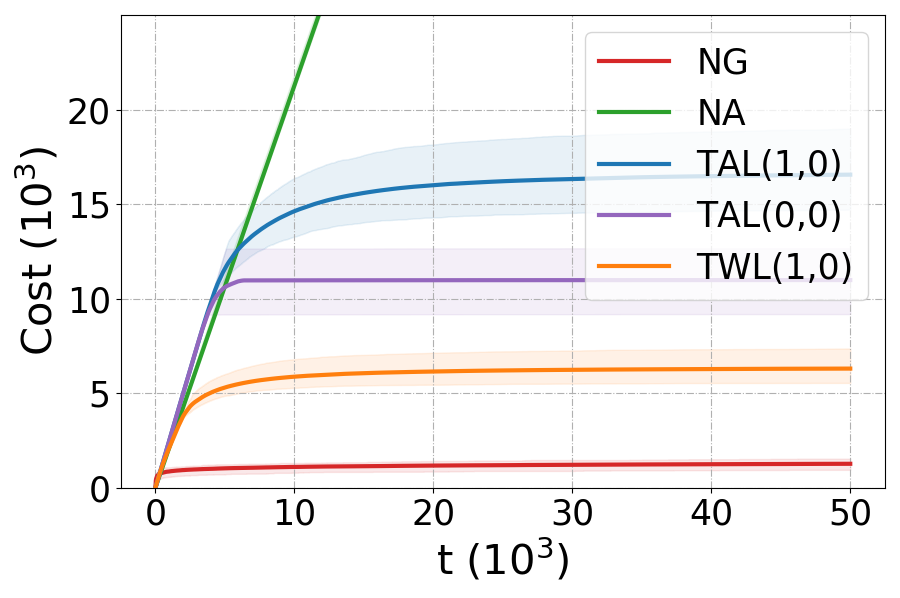}\label{fig:epsilon_cost}}
	\subfigure[TS: regrets.]{ \includegraphics[width=0.23\linewidth]{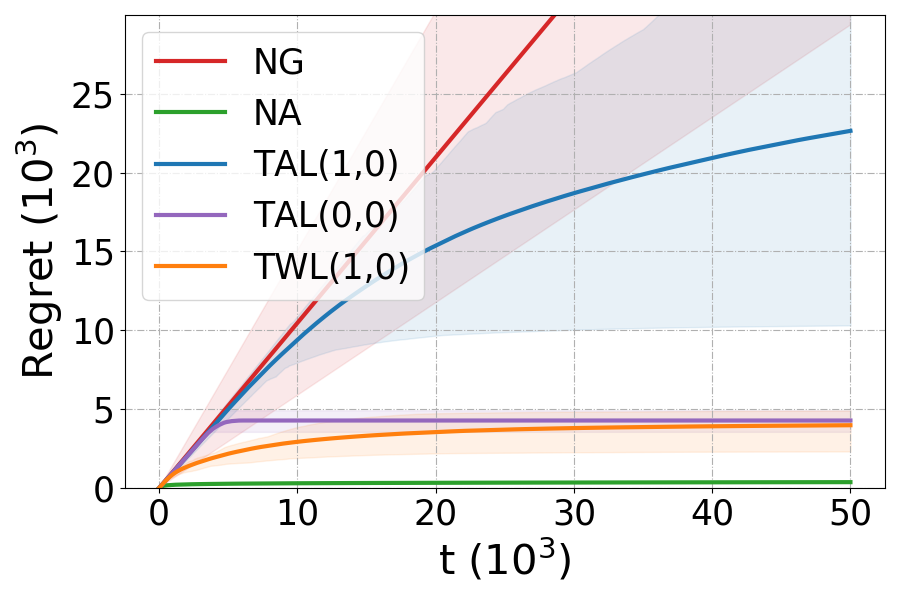}\label{fig:ts_regret}}
	\subfigure[TS: costs.]{ \includegraphics[width=0.23\linewidth]{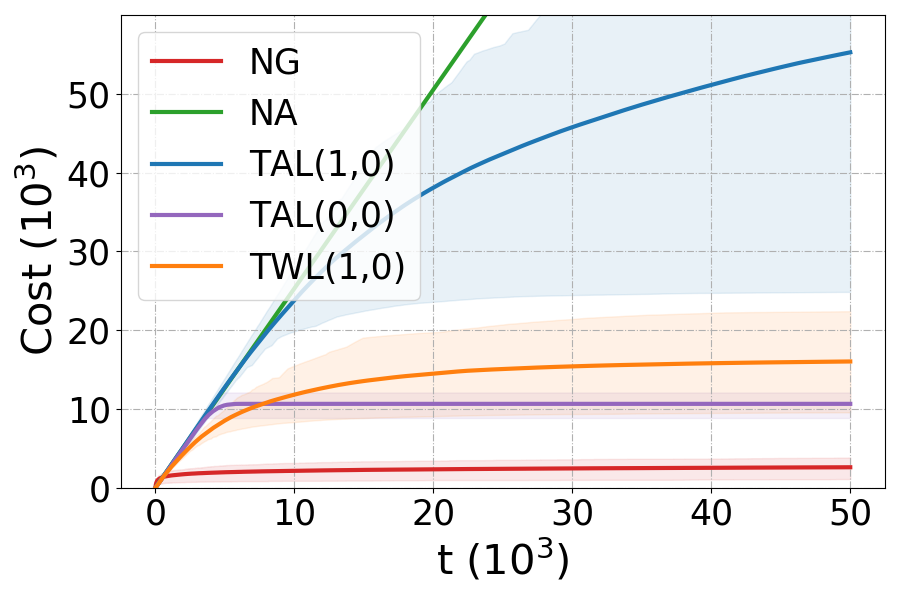}\label{fig:ts_cost}}
    \vspace{-0.1in}
    \subfigure[Mixed: regrets.]{\includegraphics[width=0.23 \linewidth]{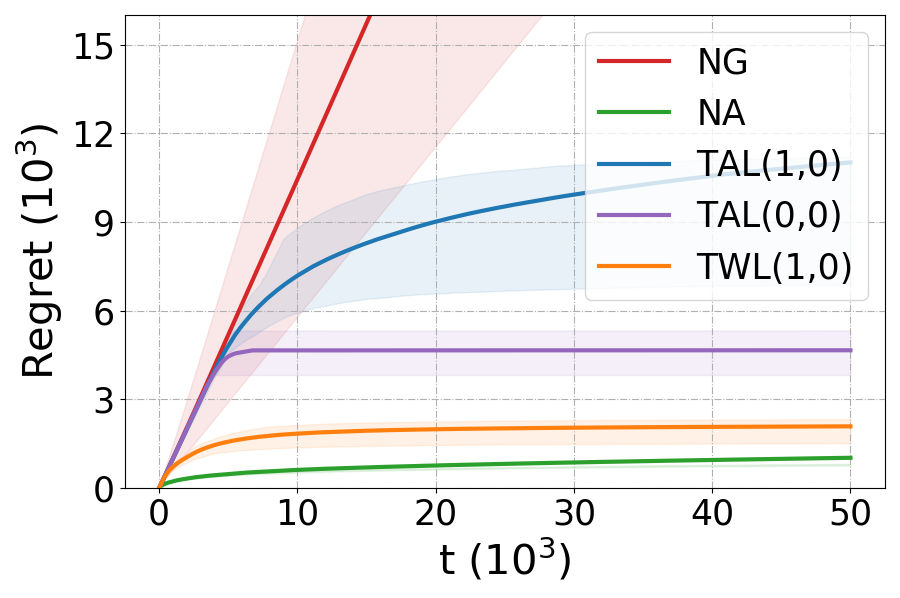}\label{fig:mixed_regret}}
	\subfigure[Mixed: costs.]{ \includegraphics[width=0.23 \linewidth]{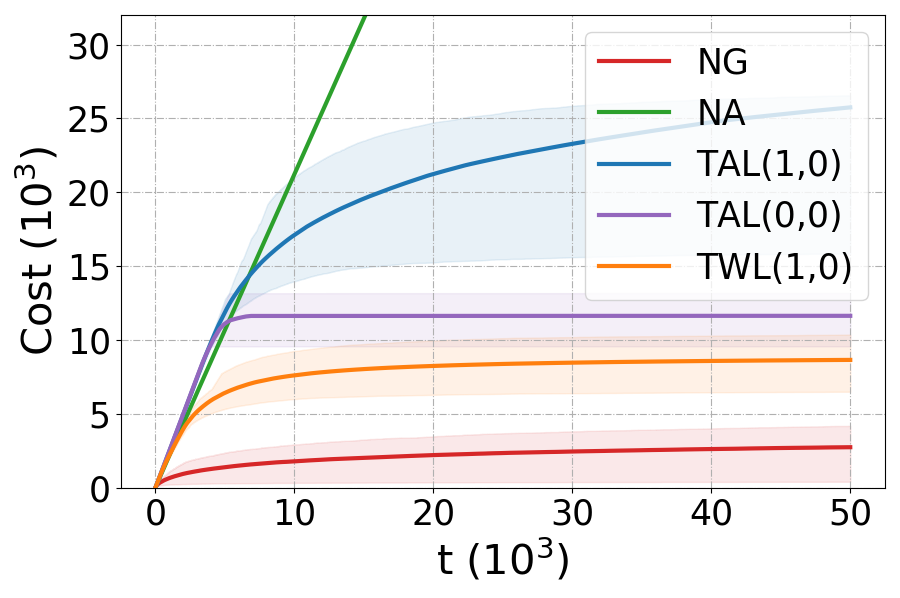}\label{fig:mixed_cost}}
    \subfigure[UCB1: random insts.]{ \includegraphics[width=0.23\linewidth]{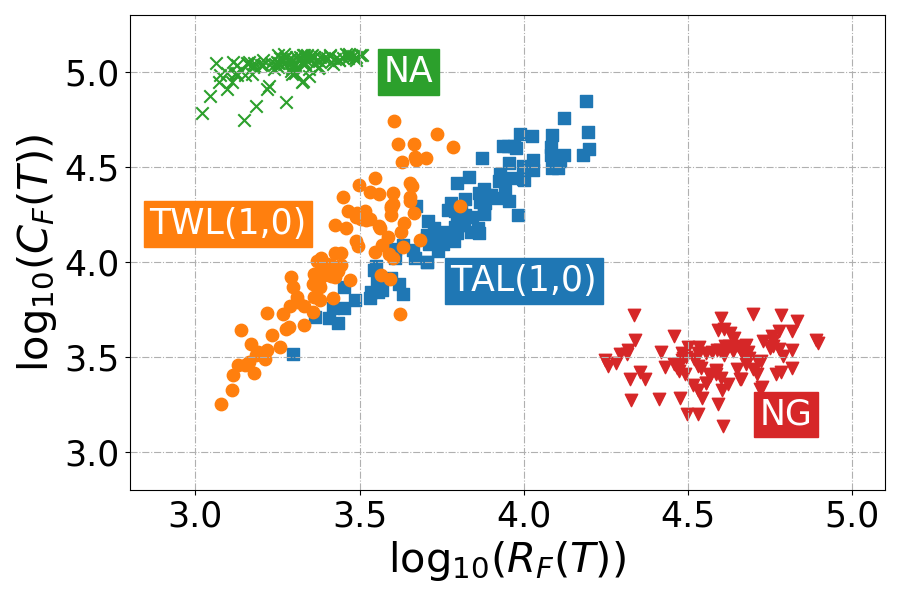}\label{fig:ucb_random}}
    \subfigure[$\varepsilon$-greedy: random insts.]{ \includegraphics[width=0.23\linewidth]{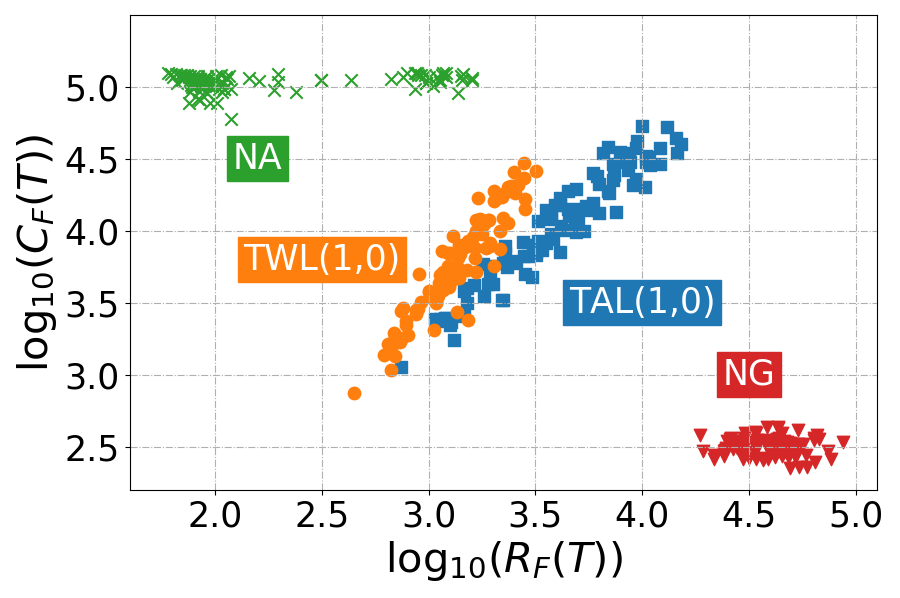}\label{fig:epsilon_random}}
    \subfigure[TS: random insts.]{ \includegraphics[width=0.23\linewidth]{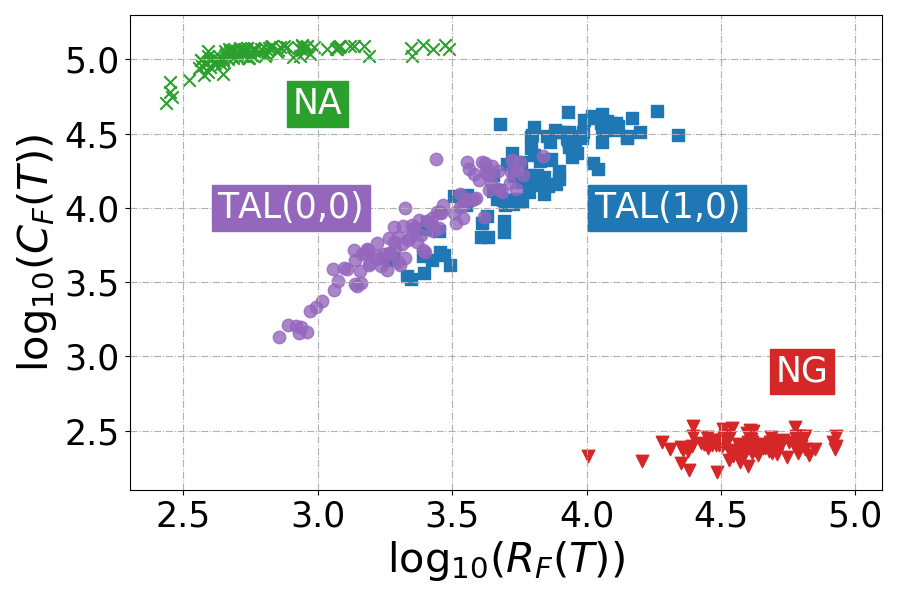}\label{fig:ts_random}}
    \subfigure[Mixed: random insts.]{ \includegraphics[width=0.23\linewidth]{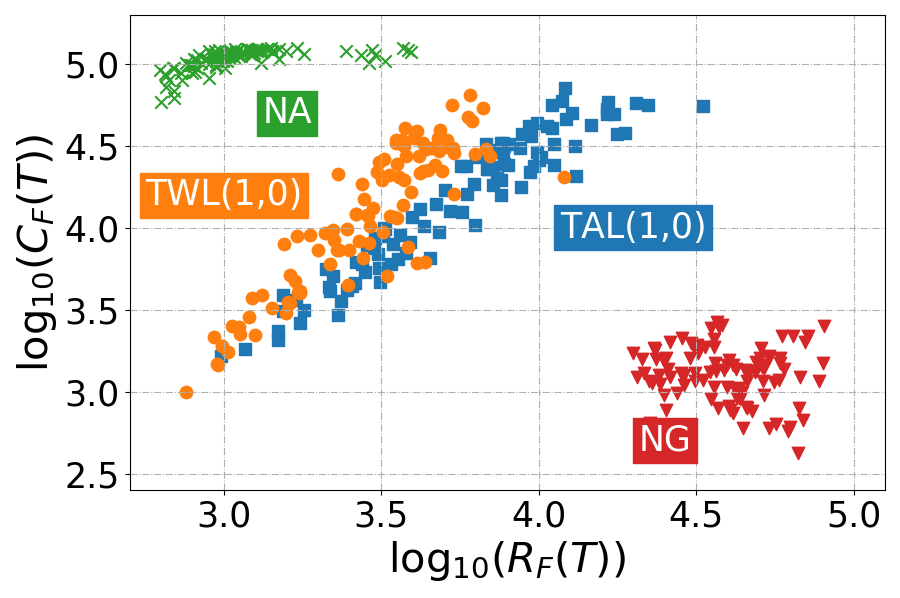}\label{fig:mix_random}}
	\caption{Experimental results on synthetic datasets with clients running UCB1, $\varepsilon$-greedy, TS, and mixed strategies. Evaluations of (a--h) are under a fixed $5$-clients-$5$-arms instance, where the curves represent the empirically averaged values and the shadowed areas represent the upper and lower $80\%$ confidence intervals. Evaluations of (i--l) are with $100$ randomly generated $5$-clients-$5$-arms instances, where each dot reports the performance (in a $\log$-$\log$ scale) under one instance and plots of a few algorithms are omitted for a better presentation here. The mixed strategies are two UCB1, two $\varepsilon$-greedy, and one TS. All time horizons are $T = 50000$.
	}
	\label{fig:performance}
\end{figure*}

\subsection{Theoretical Analysis}
The general performance of TWL can be similarly analyzed as that of TAL in Sec.~\ref{sec:tal_analysis}. The following result establishes the performance guarantee for UCB1 clients. 
\begin{theorem}[TWL with UCB1 clients]\label{thm:twl_UCB}
    For TWL with $\gamma_1 = 1$ and $\gamma_2 = 0$, if all clients  run UCB1 locally and $\mu_{\dagger,m}\neq 0$ for all $m\neq [M]$, it holds that
     \begin{align*}
		R^{\textup{TWL}}_F(T) &= O\bigg(\sum_{m\in[M]}\sum_{k\neq k_{\dagger}}\bigg[\frac{\log(T)}{\mu_{\dagger,m}M\Delta_{k}}+\frac{\Delta_{k}\log(T)}{\mu_{\dagger,m}^2}\bigg]\bigg),\\
		C^{\textup{TWL}}_F(T) &= O\bigg(\sum_{m\in[M]}\sum_{k\in [K]}\frac{(1-\mu_{k,m})\log(T)}{M\Delta^2_{k}}\\
  &+\sum_{m\in[M]}\sum_{k\neq k_{\dagger}}\bigg[\frac{\mu_{k,m}\log(T)}{\mu_{\dagger,m}M\Delta^2_{k}}+\frac{\mu_{k,m}\log(T)}{\mu_{\dagger,m}^2}\bigg]\bigg).
	\end{align*}
\end{theorem}

{The above guarantees can be interpreted in similar ways as those of TAL in Thm.~\ref{thm:tal_UCB}, i.e., one part from learning the global optimal arm and the other part from guiding agents towards it. More importantly, it is noted that with UCB1 clients, TWL strictly outperforms TAL w.r.t. both criteria since the dependency on the minimum gap $\Delta_{\min}$ is replaced by arm-dependent gaps $\Delta_{k}\geq \Delta_{\min}$, which comes precisely from its adaptive design.}

\begin{remark}
    For $\varepsilon$-greedy clients, with $\gamma_1 = \gamma_2 = 0$, the same performance guarantee as Thm.~\ref{thm:tal_epsilon} can be established for TWL because the active and non-active arms are not distinctly treated under this specification, which degrades TWL to TAL. However, experimental results show that better empirical performance is achieved with $\gamma_1 = 1$ and $\gamma_2 = 0$, whose theoretical analyses are left open for future works.
\end{remark}

{
\begin{remark}
    While TWL improves the regret of TAL regarding the dependency on $\Delta_k$, it is unclear whether its dependencies on parameters other than $M$ and $T$ are tight. One the one hand, as mentioned in Remark~\ref{rmk:lower}, a refined lower bound would be instructive in evaluating such tightness. On the other hand, it is equally worth exploring whether a refined upper bound can be obtained, which is left for further investigations.
\end{remark}
}

\section{Experimental Results}\label{sec:exp}
In this section, the proposed algorithms are empirically evaluated against two baselines, NG and NA from Sec.~\ref{sec:warm}, to demonstrate their superiority and generality. 

\subsection{Synthetic Dataset}
First, experimental results with synthetic datasets are reported in Fig~\ref{fig:performance}. In particular, two sets of experiments are performed: (1) the first environment is a fixed instance with $M=5$ clients and $K = 5$ arms, where each client's local model is specified (left to right: arm 1 to arm 5) with the following mean rewards: Client 1--$[0.2, 0.9, 0.1, 0.8, 0.6]$, Client 2--[0.4, 0.1, 0.9, 0.4, 0.8], Client 3--[0.2, 0.2, 0.5, 0.5, 0.9], Client 4--[0.4, 0.3, 0.8, 0.9, 0.4], Client 5--[0.3, 0.5, 0.2, 0.4, 0.8].
The corresponding global game then has the following mean rewards with a gap of $\Delta_{\min} = 0.1$ (left to right: arm 1 to arm 5): [0.3, 0.4,  0.5,  0.6,  0.7];
(2) the second setting is $100$ randomly generated instances with $M=5$ clients and $K = 5$ arms. Especially, the mean reward of each local arm for each client is sampled from a uniform distribution in $[0,1]$. The obtained results from these two sets of environments are reported with different client strategies in Figs.~\ref{fig:ucb_regret}--\ref{fig:mixed_cost} and Figs.~\ref{fig:ucb_random}--\ref{fig:mix_random}, respectively, and discussed in the following.
To facilitate presentations, we denote $\text{TAL}(\gamma_1, \gamma_2)$ (resp. $\text{TWL}(\gamma_1, \gamma_2)$) as TAL (resp. TWL) with specific parameters $\gamma_1$ and $\gamma_2$. {We note that with the randomly generated instances in the second environment, the reported observations are sufficiently general.} 

\textbf{UCB1 clients.} First, with UCB1 clients, from Figs.~\ref{fig:ucb_regret} and \ref{fig:ucb_cost}, it can be observed that the proposed algorithms are capable of converging while the superiority of TWL over TAL is verified. However, as claimed in Sec.~\ref{sec:warm}, the baselines are at two extremes: NG (resp. NA) is almost linear in regret (resp. cost) although performing well w.r.t. cost (resp. regret).
Fig.~\ref{fig:ucb_random} further demonstrates that TAL and TWL strike a balance between regret and cost, while the advantage of TWL is evident again. In particular, their performance scatter plots from $100$ randomly generated instances are concentrated in the diagonal between the two axes. However, the plots of the two baselines are near one axis but far from the other. 

\textbf{$\varepsilon$-greedy clients.} Figs.~\ref{fig:epsilon_regret} and \ref{fig:epsilon_cost} report that TAL and TWL can successfully teach $\varepsilon$-greedy clients with a reasonably low regret and cost at the same time. Somewhat unexpectedly, $\text{TWL}(1,0)$ has a better performance even over the theoretically sound $\text{TAL}(0,0)$ (equivalently, $\text{TWL}(0,0)$), which warrants further investigations with $\varepsilon$-greedy clients. Fig.~\ref{fig:epsilon_random} verifies that the above observations hold in general.

\textbf{TS clients.} Although not theoretically studied, Figs.~\ref{fig:ts_regret}, \ref{fig:ts_cost} and \ref{fig:ts_random} report the performances of the proposed algorithms with TS clients. While converging, the performance of $\text{TAL}(1,0)$ and $\text{TWL}(1,0)$ are highly unstable, which verifies the imbalanced exploration of TS discussed in Sec.~\ref{subsec:tal_analysis_beyond}. On the other hand, $\text{TAL}(0,0)$ has stable and competitive behaviors.

\textbf{Mixed clients.} Beyond one single local strategy, TAL and TWL are also tested with mixed strategies for clients. Especially, with two UCB1 clients, two $\varepsilon$-greedy clients, and one TS client, the results are reported in Figs.~\ref{fig:mixed_regret}, \ref{fig:mixed_cost} and \ref{fig:mix_random}. It can be observed that the proposed designs are capable of effectively guiding the clients to the global optimal arm in the face of mixed client strategies while achieving a good balance between regret and cost. These results further demonstrate the broad applicability of the designs and their appealing property of being client-strategy-agnostic.

{
\subsection{Real-world Dataset}
To further complement the observations obtained from synthetic datasets, the empirical performances of the proposed designs are further evaluated on the MovieLens dataset \cite{harper2015movielens}. The available users and movies in the dataset are both randomly divided into 15 groups to form an FMAB environment with 15 clients and 15 arms. The average movie ratings from each group of users are used to construct their local rewards. Also, the clients are considered to adopt mixed bandit strategies: 5 clients using each choice of UCB1, $\varepsilon$, and TS, respectively. From the results reported in Fig.~\ref{fig:movielens}, the aforementioned key observation is further verified that the proposed designs, i.e., TAL and TWL, are capable of effectively
guiding the clients towards the global optimal arm with a reasonable amount of adjustment cost, i.e., balancing between regret and cost. These results further demonstrate the practicability of the proposed designs.
\begin{figure}[tbh]
	\centering
	\subfigure[Mixed: regrets.]{ \includegraphics[width=0.46\linewidth]{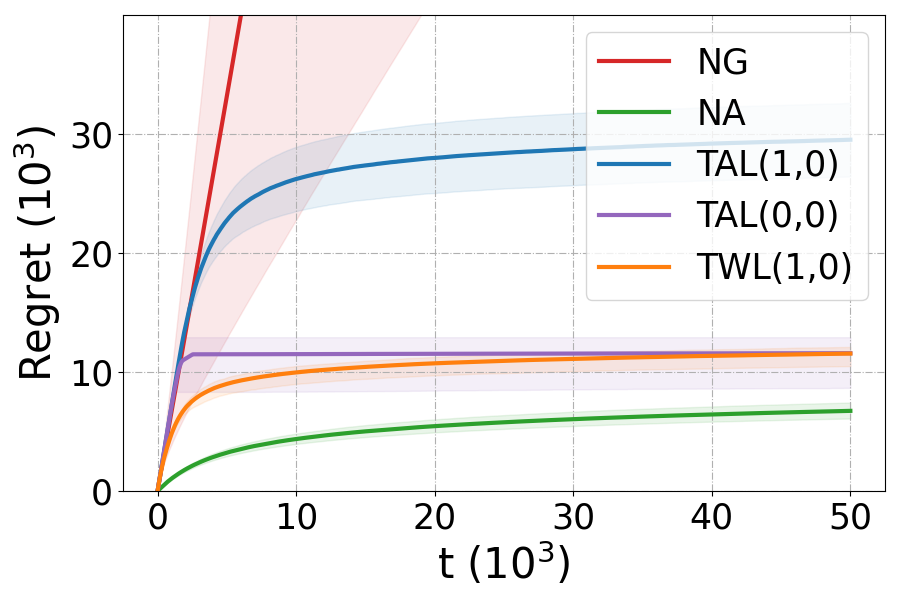}\label{fig:movielens_regret}}
	\subfigure[Mixed: costs.]{ \includegraphics[width=0.46\linewidth]{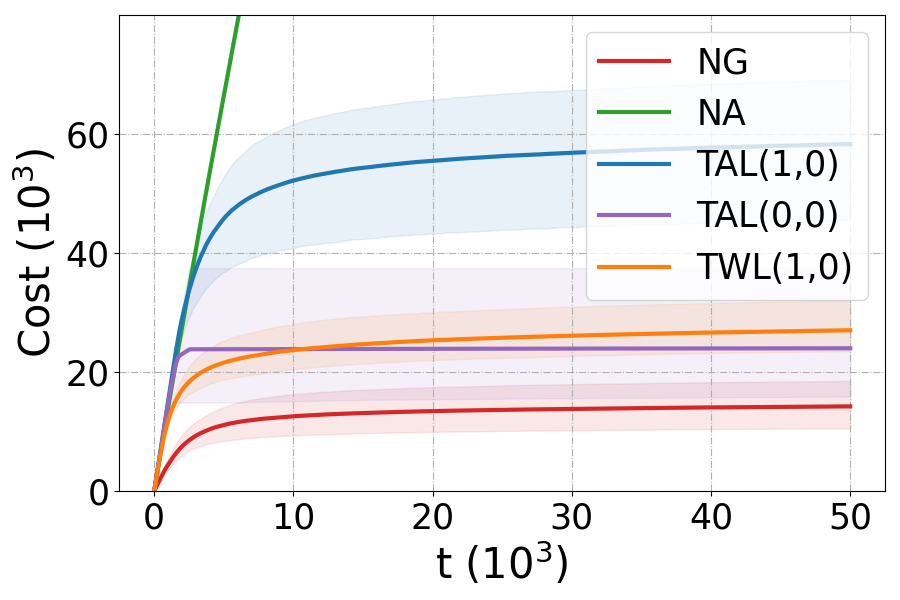}\label{fig:movielens_cost}}
 \caption{{Experimental results on the real-world MovieLens dataset with clients running mixed strategies. Evaluations (a) and (b) are under a fixed $15$-clients-$15$-arms instance, which is extracted by grouping users and movies in the MovieLens dataset. The curves represent the empirically averaged values and the shadowed areas represent the upper and lower $80\%$ confidence intervals. }}
 \label{fig:movielens}
 \end{figure}
}

\section{Conclusions}
A novel idea of reward teaching was proposed to have the server guide autonomous clients in an unknown FMAB environment via reward adjustments, which avoids any changes to the clients' protocols and removes the previous requirement of naive clients in FMAB. Two client-strategy-agnostic algorithms, TAL and TWL, were proposed. The TAL algorithm was designed with two phases to separately encourage and discourage explorations. The TWL algorithm further optimized the performance by breaking the non-adaptive phased structure into a flexible interleaving scheme. General performance analysis was established for TAL when the clients' strategies satisfy certain requirements. Especially, for the representative UCB1 and $\varepsilon$-greedy clients, rigorous analyses showed that TAL strikes a balance between regret and adjustment cost (logarithmic in both metrics), which is order-optimal w.r.t. the natural lower bound. Moreover, the analyses also demonstrated that TWL achieves an improved dependency on the sub-optimality gap than TAL due to its adaptive design. Experimental results further demonstrated the effectiveness and efficiency of the proposed algorithms. Under the reward teaching framework, many interesting questions were left open for further investigations, e.g., theoretical analysis on TAL and TWL with Thompson sampling clients.

\appendices

\section{TAL: Performance Analysis}\label{app:tal}
\subsection{General Analysis: Theorem~\ref{thm:tal_general}}
First, the following good event is established to demonstrate the effectiveness of the proposed confidence bounds.
\begin{lemma}\label{lem:tal_good_event}
	Denoting event $\Ec_F$ as
	\begin{align*}
		\Ec_F: = \left\{\forall \psi\leq T, \forall k\in [K], \left|\hat{\nu}_{k}(\psi)-\nu_{k}\right|\leq 2^{-\psi - 2}\right\}
	\end{align*}
	where $\hat{\nu}_{k}(\psi):= \frac{1}{M}\sum_{m=1}^M\hat{\mu}_{k,m}(\psi)$,
	it holds that $
		\Pb\left(\Ec_F\right) \geq 1-1/T$.
\end{lemma}
\begin{proof}
	With Hoeffding's inequality and the design that
    \begin{align*}
        \hat{\nu}_{k}(\psi) = \sum\nolimits_{m=1}^M\sum\nolimits_{\tau=F(\psi-1)+1}^{F(\psi)}X_{k,m}(N^{-1}_{k,m}(\tau))/(Mf(\psi)),
    \end{align*}
    at epoch $\psi$, for arm $k$, we have
	\begin{align*}
		\Pb\left(\left|\hat{\nu}_{k}(\psi)-\nu_{k}\right|> 2^{-\psi - 2}\right) &\leq 2\exp(-2Mf(\psi)2^{-2\psi-4})\\
        &= 1/(KT^2).
	\end{align*}
	With a union bound over $\psi \leq T$ and $k \in [K]$, the lemma can be proved.
\end{proof}

\begin{lemma}[Learning Phase in TAL; Restatement of Lemma~\ref{lem:tal_learning_general}]\label{lem:tal_learning_general_restate}
     If $\Pi_m$ is $([K], \gamma_{1}, \underline{\eta}_{m}, \overline{\eta}_m)$-sufficiently exploring for all $m\in [M]$, 
    with probability (w.p.) at least $1- 1/T$, the learning phase ends with $k_\ddagger = k_\dagger$ by time step $T_1$, and the regret and cost in the learning phase of TAL are bounded, respectively, as
    \begin{align*}
        R_{F,1}(T)\leq &\sum\nolimits_{m\in [M]}\sum\nolimits_{k\neq k_\dagger} \Delta_k \cdot \overline{\eta}_{m}\left(T_1; \gamma_1, [K]\right);\\
        C_{F,1}(T)\leq&\sum\nolimits_{m\in [M]}\sum\nolimits_{k\in [K]} \delta_{k,m}(\gamma_1) \cdot \overline{\eta}_{m}\left(T_1; \gamma_1, [K]\right),
    \end{align*}
    where $T_1 \leq \max_{m\in [M]}\left\{\underline{\eta}_{m}^{-1}\left(F(\psi_{\max}) ; \gamma_1, [K]\right)\right\}$.
\end{lemma}

\begin{proof}[Proof of Lemma~\ref{lem:tal_learning_general_restate}]
    With event $\Ec_F$ in Lemma~\ref{lem:tal_good_event} happening, we assume the learning phases end at time step $T_1$ such that
    \begin{equation*}
        T_1 \geq \max\nolimits_{m\in [M]}\left\{\underline{\eta}_{m}^{-1}\left(F(\psi_{\max}) ; \gamma_1, [K]\right)\right\}.
    \end{equation*}
    Since each local algorithm $\Pi_m$ is $([K], \gamma_{1}, \underline{\eta}_{m}, \overline{\eta}_m)$-sufficiently exploring and the rewards on all arms are constant $\gamma_1$'s, it holds that $N_{k,m}(T_1) \geq F(\psi_{\max}), \forall k\in [K], \forall m\in [M]$,
    which means that epoch $\psi_{\max}$ is reached. Thus, the confidence bound can be bounded as $\text{CB}(\psi_{\max})\leq \frac{1}{4} \cdot 2^{-\psi_{\max}} \leq \frac{1}{4} \Delta_{\min}$,
    which results in
    \begin{align*}
        \text{LCB}_{\dagger}(\psi_{\max}) &=  \hat\nu_{\dagger}(\psi_{\max}) - \text{CB}(\psi_{\max})\\
        &\geq  \nu_{\dagger} - 2\text{CB}(\psi_{\max})\geq \nu_{\dagger}-\frac{\Delta_{\min}}{2} \\
        &\geq \nu_{k}+\frac{\Delta_{\min}}{2}\geq   \nu_{k}+2\text{CB}(\psi_{\max})\\
        &\geq \hat\nu_{k}(\psi_{\max})+\text{CB}(\psi_{\max}) \\
        &=\text{UCB}_{k}(\psi_{\max}), \qquad \forall k \neq k_{\dagger}.
    \end{align*}
    Thus, the learning phase should already end. Similarly, it can be obtained that arm $k_{\dagger}$ would not be dominated by any other arm; thus $k_{\ddagger} = k_{\dagger}$.
    Then, with the observation that
    \begin{equation*}
        N_{k,m}(T_1) \leq \overline{\eta}_m(T_1; \gamma_1, [K]), \qquad \forall k\in [K], m\in [M],
    \end{equation*}
    the lemma can be proved.
\end{proof}

\begin{lemma}[Teaching Phase in TAL; Restatement of Lemma~\ref{lem:tal_teaching_general}]\label{lem:tal_teaching_general_restate}
    If the event in Lemma~\ref{lem:tal_learning_general_restate} occurs, the regret and cost in the teaching phase of TAL are bounded, respectively, as
    \begin{align*}
        R_{F,2}(T)\leq&\sum_{m\in [M]}\max_{H_m\in \Hc_m}\sum_{k\neq k_\dagger}\Delta_k \cdot \iota_{k}(T; H_m, \Bc_m, \Pi_m);\\
        C_{F,2}(T)\leq& \sum_{m\in [M]}\max_{H_m\in \Hc_m}\sum_{k\neq k_\dagger}\delta_{k,m}(\gamma_2)\cdot \iota_{k}(T; H_m, \Bc_m, \Pi_m),
    \end{align*}
    where $\Bc_m$ denotes an environment with constant rewards as $\gamma_2$ for arm $k \neq k_\dagger$ and stochastic rewards with expectation $\mu_{\dagger,m}$ for arm $k_\dagger$. The set $\Hc_m$ is defined with each element of it as a reward sequence $H_m = \{H_{k,m} : k\in [K]\}$ where $H_{k,m}\in \{\{\gamma_1\}^\tau: \tau \in [\underline{\eta}_m(T_1; \gamma_1, [K]), \overline{\eta}_m(T_1; \gamma_1, [K])]\}$.
\end{lemma}
\begin{proof}[Proof of Lemma~\ref{lem:tal_teaching_general_restate}]
    This lemma can be obtained by realizing that if the event in Lemma~\ref{lem:tal_learning_general} happens, at the beginning of the teaching phase, i.e., time step $T_1$, client $m$ has observed constant reward $\gamma_1$ on each arm $k\in [K]$ for at least $\underline{\eta}_m(T_1; \gamma_1, [K])$ and at most $\overline{\eta}_m(T_1; \gamma_1, [K])$ times, which leads to the definition of $\Hc_{m}$.

    Starting at time step $T_1$, the local bandit algorithm $\Pi_m$ can be viewed as interacting with environment $\Bc_m$ with prior input $H_m \in \Hc_m$. By recognizing that with the reward sequence $H_m\in \Hc_m$, $
        \Eb[N_{k,m}(T - T_1)|H_m, \Bc_m] \leq \Eb[N_{k,m}(T)|H_m, \Bc_m] = \iota_{k,m}(T; H_m, \Bc_m, \Pi_m)$,
    the lemma can be proved.
\end{proof}

\begin{theorem}[Overall Performance of TAL; Restatement of Theorem~\ref{thm:tal_general}]\label{thm:tal_general_restate}
    Under the assumption in Lemma~\ref{lem:tal_learning_general_restate}, with $R_{F,1}(T), C_{F,1}(T)$ defined in Lemma~\ref{lem:tal_learning_general_restate} and $R_{F,2}(T), C_{F,2}(T)$ in Lemma~\ref{lem:tal_teaching_general_restate}, the regret and cost of TAL are bounded, respectively, as $R_F(T) \leq R_{F,1}(T) + R_{F,2}(T) + O(M)$ and $C_F(T) \leq C_{F,1}(T) + C_{F,2}(T) + O(M)$.
\end{theorem}
\begin{proof}[Proof of Theorem~\ref{thm:tal_general}]
    When event $\Ec_F$ happens, the regret and cost can be obtained as the combination of Lemmas~\ref{lem:tal_learning_general_restate} and \ref{lem:tal_teaching_general_restate}. Otherwise, the regret and cost can be bounded linearly by $MT$. The lemma can then be proved with the guarantee that $\Pb(\Ec_F)\geq 1- 1/T$ as shown in Lemma~\ref{lem:tal_learning_general_restate}.
\end{proof}

\subsection{UCB1 Clients: Theorem~\ref{thm:tal_UCB}}\label{app:tal_ucb}
\begin{lemma}[Restatement of Lemma~\ref{lem:tal_learning_UCB}]\label{lem:tal_learning_UCB_restate}
    For any $\gamma\in [0,1]$ and set $\Ic\subseteq [K]$, UCB1 is $(\Ic, \gamma, \underline{\eta}, \overline{\eta})$-sufficiently exploring with $
        \underline{\eta}(\tau; \gamma, \Ic)= \lfloor \tau/ |\Ic| \rfloor$ and $\overline{\eta}(\tau; \gamma, \Ic) = \lceil \tau/ |\Ic| \rceil$.
\end{lemma}
\begin{proof}
    The UCB1 algorithm is defined in Sec.~\ref{subsec:tal_ucb} and the subscript $m$ is ignored in the following to denote a general UCB1 algorithm. To prove the lemma, it is essential to obtain that if at time step $t$, $\sum_{k\in \Ic} N_{k}(t) = \tau$, then $\max_{k\in \Ic}N_k(t) - \min_{k\in \Ic} N_k(t) \leq 1$. If this claim does not hold, {there exist arms} $k, k'$ such that $N_k(t) \geq N_{k'}(t) + 2$. Then, at the last time step that the arm $k$ is pulled, denoted as $t'$, it holds that
    \begin{align*}
        \mu'_{k}(t') &+ \sqrt{\frac{2\log(t')}{N_{k}(t')}} \overset{(a)}{\leq} \gamma + \sqrt{\frac{2\log(t')}{N_{k'}(t)+1}} \leq \gamma + \sqrt{\frac{2\log(t')}{N_{k'}(t)}} \\
        &\leq \gamma + \sqrt{\frac{2\log(t')}{N_{k'}(t')}} \overset{(b)}{=} \mu'_{k'}(t') + \sqrt{\frac{2\log(t')}{N_{k'}(t')}},
    \end{align*}
    where steps (a) and (b) leverages the fact that both arm $k$ and $k'$ receive reward $\gamma$'s. A contradiction is thus raised as arm $k$ would not be pulled, and the lemma is proved.
\end{proof}

Then, with Lemma~\ref{lem:tal_learning_UCB_restate}, we can observe that
\begin{align}
    T_1 &\leq K F(\psi_{\max}) = O\left(\frac{K\log(KT)}{M\Delta_{\min}^2}\right). \label{eqn:tal_ucb_T1}
\end{align}

\begin{lemma}[Restatement of Lemma~\ref{lem:tal_teaching_UCB}]\label{lem:tal_teaching_UCB_restate}
    If $\gamma_1 \geq \mu_{\dagger, m} > \gamma_2$ and $\Pi_m$ is UCB1, for all $k\neq k_{\dagger}$, it holds that $
        \max_{H_m\in \Hc_m}\{\iota_{k}(T; H_m, \Bc_m, \Pi_m)\} = O\left(\frac{(\gamma_1 - \gamma_2 )T_1}{K(\mu_{\dagger,m}-\gamma_2)} + \frac{\log(KT)}{(\mu_{\dagger,m} - \gamma_2)^2}\right)$.
\end{lemma}
\begin{proof}[Proof of Lemma~\ref{lem:tal_teaching_UCB_restate}]
	For $H_m$ in the set $\Hc_m$, it contains $\tau_{k,m}$ times reward $\gamma_1$ on each arm $k\in [K]$, where $\tau_{k,m} \in [\lfloor T_1/K \rfloor, \lceil T_1/K \rceil]$. We denote $t_H = \sum_{k\in [K]}\tau_{k,m}\leq T_1 + K$ as the length of reward sequence in $H_m$ and \begin{align*}
        L_{k,m}&: = \frac{(\gamma_1 - \gamma_2) \tau_{k,m}}{4(\mu - \gamma_2)^2} + \frac{2\log(T+t_H)}{4(\mu - \gamma_2)^2}\\
        &\leq \frac{(\gamma_1 - \gamma_2) (T_1/K + 1)}{4(\mu - \gamma_2)^2} + \frac{2\log(T+T_1 + K)}{4(\mu - \gamma_2)^2}.
    \end{align*} 
    It holds that
	\begin{align*}
		&\iota_{k,m}(T; H_m, \Bc_m, \Pi_m)\leq  L_{k,m}(T)\\
  &+ \Eb\left[\sum_{t\in [T]} \oneb\left\{\pi_m(t)=k, N_{k,m}(t-1) > L_{k,m}(T)\right\}\right]\\
		&= L_{k,m}(T) + \sum_{t\in [T]}\Pb\left(\pi_m(t)=k, N_{k,m}(t-1) > L_{k,m}(T)\right).
	\end{align*}
	With UCB1 as $\Pi_m$, it further holds that
	\begin{align*}
	    &\Pb\left(\pi_m(t)=k, N_{k,m}(t-1) > L_{k,m}(T)\right)\\
	    &\leq \Pb\left(\hat\mu'_{k,m}(t-1)+\sqrt{\frac{2\log(t + t_H)}{\tau_{k,m}+ N_{k,m}(t-1)}}\geq \hat\mu'_{\dagger,m}(t-1)+\right.\\
        &\left.\sqrt{\frac{2\log(t + t_H)}{\tau_{\dagger,m} + N_{\dagger,m}(t-1)}}, N_{k,m}(t-1) > L_{k,m}(T)\right)
	\end{align*}
	where the last inequality is due to a union bound.
 
    Let us separately consider $N_{k,m}(t-1) = n_{k,m}\in [L_{k,m}(T), t]$ and $N_{\dagger,m}(t-1) = n_{k,m}\in [t]$. It holds that
    \begin{align*}
        &\Pb\left(\hat\mu'_{\dagger,m}(t-1)+\sqrt{\frac{2\log(t + t_H)}{\tau_{\dagger,m } + n_{\dagger,m}}} \leq  \mu_{\dagger,m} \right)\\
        & = \Pb\left( \frac{\tau_{\dagger,m} \gamma_1 + \sum_{\tau = 1}^{n_{\dagger,m}}X_{\dagger,m}^\tau}{\tau_{\dagger,m} + n_{\dagger,m}}+\sqrt{\frac{2\log(t + t_H)}{\tau_{\dagger,m} + n_{\dagger,m}}}\leq  \mu_{\dagger,m} \right)\\
        & \overset{(a)}{\leq} \Pb\left(\frac{\tau_{\dagger,m}  \mu_{\dagger,m} + \sum_{\tau = 1}^{n_{\dagger, m}}X_{\dagger,m}^{\tau}}{\tau_{\dagger,m} + n_{\dagger,m}}+\sqrt{\frac{2\log(t + t_H)}{\tau_{\dagger,m} + n_{\dagger,m}}}\leq  \mu_{\dagger,m} \right)\\
        & \overset{(b)}{\leq} \frac{1}{(t+t_H)^4},
    \end{align*}
    where inequality (a) is an essential step of ``optimism'' due to $\gamma_1 \geq \mu_{\dagger,m}$ and inequality (b) holds is from Hoeffding's inequality.
    Also, it can be observed that with $n_{k,m}\geq L_{k,m}(T)$, 
	\begin{align*}
	    &\hat\mu'_{k,m}(t-1)+\sqrt{\frac{2\log(t + t_H)}{\tau_{k,m} + n_{k,m}}}\\
     &=  \frac{\tau_{k,m}\cdot \gamma_1 + n_{k,m} \cdot \gamma_2 }{\tau_{k,m} + n_{k,m}}+\sqrt{\frac{2\log(t + t_H)}{\tau_{k,m} + n_{k,m}}}\leq \mu_{\dagger,m}.
	\end{align*}

    Thus, with a union bound, it holds that
    \begin{align*}
        &\Pb\left(\pi_m(t)=k, N_{k,m}(t-1) > L_{k,m}(T)\right) \\
        &\leq \sum\nolimits_{n_{k,m} = L_{k,m}(T)}^t \sum\nolimits_{n_{\dagger,m}\in [t]} \frac{1}{(t+t_H)^4} \leq \frac{1}{t^2}
    \end{align*}
	It is then indicated that
	\begin{align*}
		\Eb\left[N_{k,m}(T)\right]\leq L_{k,m}(T) + \sum\nolimits_{t=1}^T\frac{1}{t^2} \leq L_{k,m}(T) + 2,
	\end{align*}
    which proves  Lemma~\ref{lem:tal_teaching_UCB_restate}.
\end{proof}

\begin{theorem}[TAL with UCB1 clients; Restatement of Theorem~\ref{thm:tal_UCB}]\label{thm:tal_UCB_restate}
    For TAL with $\gamma_1 = 1$ and $\gamma_2 = 0$, if all clients run  UCB1 locally and $\mu_{\dagger,m}\neq 0$ for all $m\in [M]$, it holds that
    \begin{align*}
           R_F(T) &= O\bigg(\sum_{m\in[M]}\sum_{k\neq k_{\dagger}}\bigg[\frac{\Delta_{k}\log(KT)}{\mu_{\dagger,m}M\Delta_{\min}^2}+\frac{\Delta_{k}\log(KT)}{\mu_{\dagger,m}^2}\bigg]\bigg);\\
		C_F(T) &= O\bigg(\sum_{m\in[M]}\sum_{k\in [K]}\frac{(1-\mu_{k,m})\log(KT)}{M\Delta^2_{\min}}\\
  &+\sum_{m\in[M]}\sum_{k\neq k_{\dagger}}\bigg[\frac{\mu_{k,m}\log(KT)}{\mu_{\dagger,m}M\Delta^2_{\min}}+\frac{\mu_{k,m}\log(KT)}{\mu_{\dagger,m}^2}\bigg] \bigg).
       \end{align*}
\end{theorem}
\begin{proof}[Proof of Theorem~\ref{thm:tal_UCB_restate}]
    From Eqn.~\eqref{eqn:tal_ucb_T1}, it holds that $
        T_1 = O\left(\frac{K\log(KT)}{M\Delta_{\min}^2}\right)$.
    With Lemmas~\ref{lem:tal_learning_general_restate} and \ref{lem:tal_learning_UCB_restate}, it holds that
    \begin{align*}
        R_{F,1}(T)&\leq \sum\nolimits_{m\in [M]}\sum\nolimits_{k\neq k_\dagger} \Delta_k \cdot \overline{\eta}_{m}\left(T_1; \gamma_1, [K]\right) \\
        &= O\left(\sum\nolimits_{m\in [M]}\sum\nolimits_{k\neq k_\dagger}\frac{\Delta_k\log(KT)}{M\Delta_{\min}^2}\right);\\
        C_{F,1}(T)&\leq\sum\nolimits_{m\in [M]}\sum\nolimits_{k\in [K]} \delta_{k,m}(\gamma_1) \cdot \overline{\eta}_{m}\left(T_1; \gamma_1, [K]\right) \\
        &\overset{(a)}{=} O\left(\sum\nolimits_{m\in [M]}\sum\nolimits_{k\in [K]}\frac{(1-\mu_{k,m})\log(KT)}{M\Delta_{\min}^2}\right),
    \end{align*}
    where equation (a) also utilizes that with $\gamma_1 = 1$, $\delta_{k,m}(\gamma_1) = 1-\mu_{k,m}$.

    Then, with Lemmas~\ref{lem:tal_teaching_general_restate} and \ref{lem:tal_teaching_UCB_restate}, it holds that
    \begin{align*}
        &R_{F,2}(T)\leq\sum_{m\in [M]}\max_{H_m\in \Hc_m}\sum_{k\neq k_\dagger}\Delta_k \cdot \iota_{k}(T; H_m, \Bc_m, \Pi_m) \\
        &= O\left(\sum_{m\in [M]}\sum_{k\neq k_{\dagger}}\frac{\Delta_k\log(KMT)}{\mu_{\dagger,m}M\Delta^2_{\min}} + \frac{\Delta_k\log(KMT)}{\mu_{\dagger,m}^2}\right);\\
        &C_{F,2}(T)\leq \sum_{m\in [M]}\max_{H_m\in \Hc_m}\sum_{k\neq k_\dagger}\delta_{k,m}(\gamma_2)\cdot \iota_{k}(T; H_m, \Bc_m, \Pi_m) \\
        &\overset{(a)}{=} O\left(\sum_{m\in[M]}\sum_{k\neq k_{\dagger}}\frac{\mu_{k,m}\log(KT)}{\mu_{\dagger,m}M\Delta^2_{\min}}+\frac{\mu_{k,m}\log(KT)}{\mu_{\dagger,m}^2}\right),
    \end{align*}
    where equation (a) uses the fact that with $\gamma_2 = 0$, $\delta_{k,m}(\gamma_2) = \mu_{k,m}$. The theorem is then proved.
\end{proof}

\subsection{$\varepsilon$-greedy Clients: Theorem~\ref{thm:tal_epsilon}}
\begin{lemma}\label{lem:tal_learning_epsilon_restate}
    For any $\gamma\in [0,1]$, if ties among arms are broken uniformly at random, with probability at least $1-1/T$, $\varepsilon$-greedy is $([K], \gamma, \underline{\eta}, \overline{\eta})$-sufficiently exploring with $\underline{\eta}(\tau; \gamma, [K])$ and $\overline{\eta}(\tau; \gamma, [K]) = O( \tau/K \pm \log(KT))$.
\end{lemma}
\begin{proof}[Proof of Lemma~\ref{lem:tal_learning_epsilon_restate}]
    Since the rewards on each arm are all $\gamma$ (thus the sample means are all $\gamma$) and the ties among arms are broken uniformly at random, the algorithm would pull each arm $k\in [K]$ with equal probability $1/K$. Thus, denoting the number of pulls on an arbitrary arm $k$ by time $\tau$ as $N_k(\tau)$, it holds that $
        N_k(\tau) = \sum_{t\in [\tau]}\oneb\{\pi(t) = k\}$.
    Using Bernstein's inequality and $\Pb(\pi(t) = k) = 1/K$, we can obtain that
    \begin{align*}
        \Pb\left(\left|N_k(\tau) - \frac{\tau}{K} \right|\geq x\right) \leq 2\exp\left(-\frac{x^2}{2\frac{\tau}{K} + \frac{2}{3}x}\right)\leq \frac{1}{KT},
    \end{align*}
    where $x: = \frac{\tau}{4K}+ \frac{8}{3}\log(KT)$.
    
    With a union bound over $k\in [K]$, we can obtain that $\underline{\eta}(\tau; \gamma, [K]) = \frac{3\tau}{4K} - \frac{8}{3}\log(KT)$ and $\overline{\eta}(\tau; \gamma, [K]) = \frac{5\tau}{4K} + \frac{8}{3}\log(KT)$,
    which concludes the proof.
\end{proof}

Using Lemmas~\ref{lem:tal_learning_general_restate} and \ref{lem:tal_learning_epsilon_restate}, we can obtain that with probability $1-1/T$, the learning phase of TAL ends at $T_1$ such that
\begin{align}
    T_1&\leq \max_{m\in [M]}\left\{\underline{\eta}_{m}^{-1}\left(F(\psi_{\max}) ; \gamma_1, [K]\right)\right\} \notag \\
    &= \frac{4K}{3}F(\psi_{\max}) + \frac{32K}{9}\log(KMT) \notag\\
    &= O\left(\frac{K\log(KT)}{M\Delta_{\min}^2} + K\log(KMT)\right) \label{eqn:tal_epsilon_T1}.
\end{align}

\begin{lemma}\label{lem:tal_teaching_epsilon_restate}
    If $\Pi_m$ is $\varepsilon$-greedy and $\mu_{\dagger,m} >\gamma_1 = \gamma_2 = 0$, with probability at least $1-1/T$, it holds that $
        \max_{H_m\in \Hc_m}\left\{\sum\nolimits_{k\neq k_\dagger}\iota_{k,m}(T; H_m, \Bc_m, \Pi_m)\right\} = O\left(\frac{K\log(T)}{\mu_{\dagger,m}^2}\right)$.
\end{lemma}
\begin{proof}[Proof of Lemma~\ref{lem:tal_teaching_epsilon_restate}]
    Since $\gamma_1  = 0$, the reward sequence $H_{m}\in \Hc_m$ are all zeros. Further with $\gamma_2 = 0$, once a non-zero reward is received on arm $k_{\dagger}$, it will immediately have the highest sample mean.
    First, if arm $k_{\dagger}$ has been played at least $n'= \left\lceil \frac{\log(2T)}{2\mu^2_{\dagger,m}}\right\rceil$ times, where
    Hoeffding's inequality indicates that with a probability of at least $1-\frac{1}{2T}$, there is at least one non-zero reward collected from arm $k_{\dagger}$.   
    
    Furthermore, when there are no non-zero rewards collected on arm $k_{\dagger}$, the arms are pulled with equal probabilities (since they all have zero as sample means due to $\gamma_2 = 0$). With Bernstein's inequality, it further holds that
    \begin{align*}
            &\Pb\left(\sum\nolimits_{t\in [\tau']}\oneb\{\pi_m(t)=k_{\dagger}\} - \frac{\tau'}{K}\leq n' - \frac{\tau'}{K}\right)\\
            &\leq \exp\left(-\left(\frac{\tau'}{K} - n'\right)^2/\left(2\frac{\tau'}{K} + \frac{2}{3}\left(\frac{\tau'}{K} - n'\right)\right)\right) \leq \frac{1}{2T},
        \end{align*}
    where $\tau' = \frac{4K\log(2T)}{3\mu_{\dagger,m}^2} + \frac{32K\log(T)}{9}  = O\left(\frac{K\log(T)}{\mu_{\dagger,m}^2}\right)$.
    
    Thus, with at most $\tau'$ steps, the arm $k_{\dagger}$ would have the highest sample mean. Afterward, the other arms will only be pulled during exploration, i.e., with probability $\varepsilon(t) = O(K/t)$, which would only result in $O(K\log(T))$ pulls in expectation. The lemma is then proved.
\end{proof}

\begin{theorem}[TAL with $\varepsilon$-greedy clients; Restatement of Theorem~\ref{thm:tal_epsilon}]\label{thm:tal_epsilon_restate}
    For TAL with $\gamma_1 = \gamma_2 = 0$, if clients run $\varepsilon$-greedy and break ties uniformly at random, and $\mu_{\dagger,m }\neq 0, \forall m\in [M]$, it holds that $R_F(T) = O\left(\left[\frac{K\Delta_{\max}}{\Delta^2_{\min}}+\sum_{m\in[M]}\frac{K\Delta_{\max}}{\mu_{\dagger,m}^2}\right]\log(KMT)\right)$ and $C_F(T)=O\left(\sum_{m\in[M]}\left[\frac{K\mu_{*,m}}{M\Delta^2_{\min}}+\frac{K\mu_{*,m}}{\mu_{\dagger,m}^2}\right]\log(KMT)\right)$.
\end{theorem}

\begin{proof}[Proof of Theorem~\ref{thm:tal_epsilon_restate}]
    From Eqn.~\eqref{eqn:tal_epsilon_T1}, with probability $1-1/T$, it holds that $
        T_1 = O\left(\frac{K\log(KT)}{M\Delta^2_{\min}} + K\log(KMT)\right)$.
    Using $\overline{\eta}(T_1; \gamma_1, [K])$ from Lemma~\ref{lem:tal_learning_epsilon_restate}, it holds that
    \begin{align*}
        &R_{F,1}(T)\leq \sum\nolimits_{m\in [M]}\sum\nolimits_{k\neq k_\dagger} \Delta_k \cdot \overline{\eta}_{m}\left(T_1; \gamma_1, [K]\right) \\
        & = O\left(\frac{K\Delta_{\max}\log(KT)}{\Delta_{\min}^2} + MK\Delta_{\max}\log(KMT)\right);\\
        &C_{F,1}(T)\leq\sum\nolimits_{m\in [M]}\sum\nolimits_{k\in [K]} \delta_{k,m}(\gamma_1) \cdot \overline{\eta}_{m}\left(T_1; \gamma_1, [K]\right) \\
        &\overset{(a)}{=} O\left(\frac{K\mu_{*,m}\log(KT)}{\Delta_{\min}^2} + MK\mu_{*,m}\log(KMT)\right),
    \end{align*}
    where step (a) leverages the fact that with $\gamma_1 = 0$, $\delta_{k,m}(\gamma_1) = \mu_{k,m}\leq \mu_{*,m}$.

    Furthermore, combining Lemmas~\ref{lem:tal_teaching_general_restate} and \ref{lem:tal_teaching_epsilon_restate}, with probability $1- 1/T$, it holds that
    \begin{align*}
        R_{F,2}(T) = O\left(\sum\nolimits_{m\in [M]}\frac{K\Delta_{\max}\log(MT)}{\mu_{\dagger,m}^2}\right);\\
        C_{F,2}(T)\overset{(a)}{=} O\left(\sum\nolimits_{m\in [M]} \frac{K\mu_{*,m}\log(MT)}{\mu_{\dagger,m}^2}\right),
    \end{align*}
    where step (a) leverages the fact that with $\gamma_2 = 0$, $\delta_{k,m}(\gamma_2) = \mu_{k,m}\leq \mu_{*,m}$. Putting these two observations into Theorem~\ref{thm:tal_general_restate}, the theorem is then proved.
\end{proof}

\section{TWL: Performance Analysis}
\begin{lemma}[Arm Elimination in TWL]\label{lem:twl_arm_elim}
    Denote event $\Ec_{G}$ as
    \begin{align*}
        \Ec_{G} = &\textup{\{each arm $k\neq k_{\dagger}$ is eliminated from the active arm set }\\
        &\textup{$\Upsilon$  in TWL by the end of epoch $\psi_k$\}},
    \end{align*}
    where $\psi_{k}: = \left\lceil\log_2(1/\Delta_{k})\right\rceil$, it holds that $\Pb(\Ec_G) \geq 1 - 1/T$.
\end{lemma}
\begin{proof}[Proof of Lemma~\ref{lem:twl_arm_elim}]
    First, similar to Lemma~\ref{lem:tal_good_event}, we can establish that with probability at least $1-1/T$, it holds that
	\begin{align*}
		\left|\hat{\nu}_{k}(\psi)-\nu_{k}\right|\leq \text{CB}(\psi) = 2^{-\psi - 2}, \qquad \forall \psi \leq T, \forall k\in \Upsilon_{\psi},
	\end{align*}
    where $\Upsilon_{\psi}$ denotes the active arm set in epoch $\psi$. 
    Based on this event, we can first observe that arm $k_{\dagger}$ would not be eliminated. Furthermore, at the end of epoch $\psi_k$, if arm $k\neq k_{\dagger}$ is not eliminated, both arm $k_{\dagger}$ and arm $k$ would be active. However, we can observe that
    \begin{align*}
        \text{LCB}_{\dagger}(\psi_{k}) &=  \hat\nu_{\dagger}(\psi_{k}) - \text{CB}(\psi_{k})\geq  \nu_{\dagger} - 2\text{CB}(\psi_{k})\\
        &\geq \nu_{\dagger}-\frac{\Delta_{k}}{2} \geq \nu_{k}+\frac{\Delta_{k}}{2}\geq   \nu_{k}+2\text{CB}(\psi_{k})\\
        & \geq \hat\nu_{k}(\psi_k)+\text{CB}(\psi_k) =\text{UCB}_{k}(\psi_k),
    \end{align*}
    which means arm $k$ should already be eliminated. 
\end{proof}

\begin{lemma}[Active Arms in TWL]\label{lem:twl_learning_UCB}
    If $\Pi_m$ is UCB1, for any $\gamma_1$ and $\gamma_2 $, with probability at least $1 - 1/T$, for all $k\neq k_\dagger$, it holds that $
        N_{k,m}^1(T) := \sum_{t\in [T]}\oneb\{\pi_m(t) = k, k \in \Upsilon(t)\}=O\left(\frac{\log(KT)}{M\Delta_k^2}\right)$,
    where $\Upsilon(t)$ denotes the active arm set at time step $t$ and $\psi(t)$ denotes the epoch index at time step $t$.
\end{lemma}
\begin{proof}[Proof of Lemma~\ref{lem:twl_learning_UCB}]
    Using the same procedure in Lemma~\ref{lem:tal_learning_UCB_restate}, with constant reward $\gamma_1$'s for active arms and constant reward $\gamma_2$'s for inactive arms, it can be observed that at the end of epoch $\psi$, each client $m$ pulls each active arm $k\in \Upsilon(\psi)$ the same $F(\psi)$ times. As arm $k\neq k_\dagger$ is eliminated from the active arm set by the end of phase $\psi_k$ based on event $\Ec_G$ introduced in Lemma~\ref{lem:twl_arm_elim}, it holds that $
        N_{k,m}^1(T)  \leq F(\psi_k) =  O\left(\frac{\log(KT)}{M\Delta_k^2}\right)$,
    which concludes the proof.
\end{proof}

\begin{lemma}[Inactive Arms in TWL]\label{lem:twl_teaching_UCB}
    If $\Pi_m$ is UCB1, for any $\gamma_1$ and $\gamma_2$ such that $\gamma_1 \geq \mu_{\dagger, m} > \gamma_2$, with probability at least $1 - 1/T$, for all $k\neq k_\dagger$, it holds that $
        N_{k,m}^2(T) := \sum_{t\in [T]}\oneb\{\pi_m(t) = k, k \notin \Upsilon(t)\}= O\left(\frac{(\gamma_1 - \gamma_2 )N^{1}_{k,m}(T)}{(\mu_{\dagger,m}-\gamma_2)} + \frac{\log(KT)}{(\mu_{\dagger,m} - \gamma_2)^2}\right)$.
\end{lemma}
\begin{proof}[Proof of Lemma~\ref{lem:twl_teaching_UCB}]
    This lemma can be established following the same procedure as Lemma~\ref{lem:tal_teaching_UCB_restate}.
\end{proof}

\begin{theorem}[TWL with UCB1 clients; Restatement of Theorem~\ref{thm:twl_UCB}]\label{thm:twl_UCB_restate}
    For TWL with $\gamma_1 = 1$ and $\gamma_2 = 0$, if all clients  run UCB1 locally and $\mu_{\dagger,m}\neq 0$ for all $m\neq [M]$, it holds that 
    \begin{align*}
		R^{\textup{TWL}}_F(T) &= O\bigg(\sum_{m\in[M]}\sum_{k\neq k_{\dagger}}\bigg[\frac{\log(T)}{\mu_{\dagger,m}M\Delta_{k}}+\frac{\Delta_{k}\log(KT)}{\mu_{\dagger,m}^2}\bigg]\bigg),\\
		C^{\textup{TWL}}_F(T) &= O\bigg(\sum_{m\in[M]}\sum_{k\in [K]}\frac{(1-\mu_{k,m})\log(KT)}{M\Delta^2_{k}}\\
  &+\sum_{m\in[M]}\sum_{k\neq k_{\dagger}}\bigg[\frac{\mu_{k,m}\log(KT)}{\mu_{\dagger,m}M\Delta^2_{k}}+\frac{\mu_{k,m}\log(KT)}{\mu_{\dagger,m}^2}\bigg]\bigg).
	\end{align*}
\end{theorem}

\begin{proof}[Proof of Theorem~\ref{thm:twl_UCB_restate}]
    From Lemma~\ref{lem:twl_learning_UCB}, with probability at least $1- 1/T$, it holds that $
        N^{1}_{k,m}(T) = O\left(\frac{\log(KT)}{M\Delta_{k}^2}\right)$,
    thus we can specify $
        N^{2}_{k,m}(T) = O\left(\frac{\log(KT)}{\mu_{\dagger,m}M\Delta_{k}^2} + \frac{\log(KT)}{\mu_{\dagger,m}^2}\right)$ with Lemma~\ref{lem:twl_teaching_UCB}.
    The overall regret and cost can then be bound as
    \begin{align*}
        &R^{\text{TWL}}_{F}(T) \leq \sum_{m\in [M]}\sum_{k\neq k_\dagger}\left(N^{1}_{k,m}(T) + N^2_{k,m}(T)\right) \Delta_{k} + \frac{MT}{T}\\
        & = O\left(\sum\nolimits_{m\in [M]}\sum\nolimits_{k\neq k_\dagger} \frac{\log(KT)}{\mu_{\dagger,m}M\Delta_{k}} + \frac{\Delta_k\log(KT)}{\mu_{\dagger,m}^2}\right);\\
        &C^{\text{TWL}}_{F,k,1}(T) \leq \sum\nolimits_{m\in [M]} \sum\nolimits_{k\in [K]}N^{1}_{k,m}(T) \cdot \delta_{k,m}(1)\\
        & + \sum\nolimits_{m\in [M]} \sum\nolimits_{k\neq k_\dagger }N^{1}_{k,m}(T) \cdot \delta_{k,m}(0) + \frac{MT}{T}\\
        & = O\left(\sum\nolimits_{m\in[M]}\sum\nolimits_{k\in [K]}\frac{(1-\mu_{k,m})\log(KT)}{M\Delta^2_{k}}\right.\\
  &\left.+\sum\nolimits_{m\in[M]}\sum\nolimits_{k\neq k_{\dagger}}\left[\frac{\mu_{k,m}\log(KT)}{\mu_{\dagger,m}M\Delta^2_{k}}+\frac{\mu_{k,m}\log(KT)}{\mu_{\dagger,m}^2}\right]\right),
    \end{align*}
    which concludes the proof.
\end{proof}

\bibliographystyle{IEEEtran}
\bibliography{bandit}

\vspace{-15 mm}
\begin{IEEEbiography}[{\includegraphics[width=1in,height=1.25in,clip,keepaspectratio]{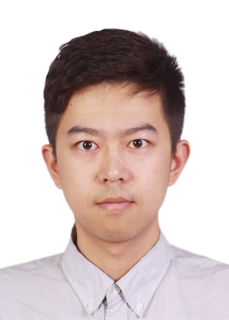}}]{Chengshuai Shi} is currently a Ph.D. student at the Charles L. Brown Department of Electrical and Computer Engineering, University of Virginia. He received his B.E. degree in Electrical Engineering from the School of the Gifted Young, University of Science and Technology of China, in 2019. His current research focuses on multi-armed bandits, federated learning, and reinforcement learning.
\end{IEEEbiography}

\vspace{15 mm}
\begin{IEEEbiography}[{\includegraphics[width=1in,height=1.25in,clip,keepaspectratio]{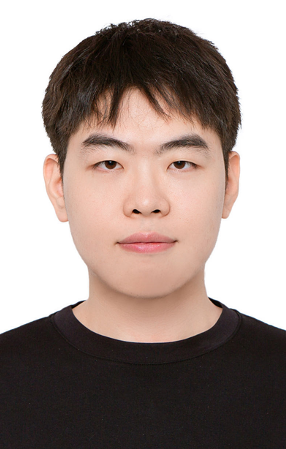}}]{Wei Xiong} is currently a Ph.D. student at the Department of Computer Science, University of Illinois Urbana-Champaign. Previously, he received a master's degree in mathematics in 2023 from The Hong Kong University of Science and Technology and a B.S. in mathematics from the University of Science and Technology of China in 2021. His current research interests focus on reinforcement learning and alignment for foundation generative models.
\end{IEEEbiography}

\begin{IEEEbiography}[{\includegraphics[width=1in,height=1.25in,clip,keepaspectratio]{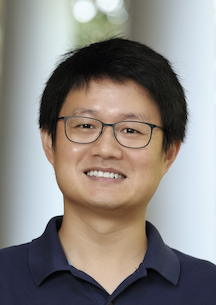}}]{Cong Shen}
(S'01-M'09-SM'15) is currently an Assistant Professor at the Charles L. Brown Department of Electrical and Computer Engineering, University of Virginia. He received his B.E. and M.E. degrees from the Department of Electronic Engineering, Tsinghua University, China, and his Ph.D. degree in Electrical Engineering from University of California, Los Angeles. He has extensive industry experience, having worked for Qualcomm, SpiderCloud Wireless, Silvus Technologies, and Xsense.ai, in various full time and consulting roles. He currently serves as an associate editor for the IEEE Transactions on Communications, IEEE Transactions on Green Communications and Networking, and IEEE Transactions on Machine Learning in Communications and Networking. He received the NSF CAREER award in 2022, and the Best Paper Award in 2021 IEEE International Conference on Communications (ICC). His general research interests are in the area of communications, wireless networks, and machine learning.
\end{IEEEbiography}

\vspace{-22 mm}
\begin{IEEEbiography}[{\includegraphics[width=1in,height=1.25in,clip,keepaspectratio]{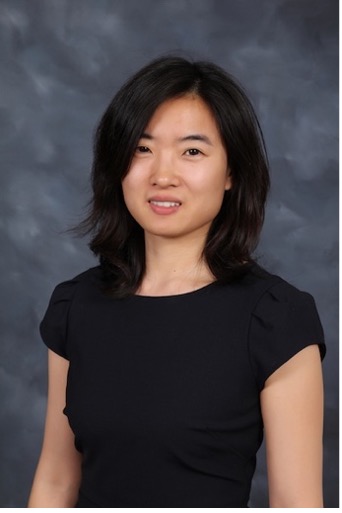}}]{Jing Yang}
(S’08-M’10-SM’23) is an Associate Professor of Electrical Engineering at the Pennsylvania State University. She received her B.S. degree from the University of Science and Technology of China (USTC), and the M.S. and Ph.D. degrees from the University of Maryland, College Park, all in Elec- trical Engineering. She received the National Science Foundation CAREER award in 2015 and the WICE Early Achievement Award in 2020, and was selected as one of the 2020 N2Women: Stars in Computer Networking and Communications. She served as a Symposium/Track/Workshop Co-chair for Asilomar 2023, ICC 2021, INFOCOM 2021-AoI Workshop, WCSP 2019, CTW 2015, PIMRC 2014, a TPC Member of several conferences, and an Editor for IEEE Transactions on Green Communications and Networking from 2017 to 2020. She is now an Editor for IEEE Transactions on Wireless Communications and IEEE Transactions on Cognitive Communications and Networking. Her research interests lie in multi-armed bandits and reinforcement learning, federated learning, and wireless communications and networking.
%She served as a Symposium/Workshop Co-chair for ICC 2021, INFOCOM 2021-AoI Workshop, WCSP 2019, CTW 2015, PIMRC 2014, a TPC Member of several conferences, and an Editor for IEEE Transactions on Green Communications and Networking from 2017 to 2020. She is now an Editor for IEEE Transactions on Wireless Communications and IEEE Transactions on Cognitive Communications and Networking. Her research interests lie in wireless communications and networking, information theory and statistical learning theory.
\end{IEEEbiography}

\end{document}